\documentclass{article}

    \usepackage[preprint]{neurips_2025}

\usepackage[utf8]{inputenc} 
\usepackage[T1]{fontenc}    
\usepackage{hyperref}       
\usepackage{url}            
\usepackage{booktabs}       
\usepackage{amsfonts}       
\usepackage{nicefrac}       
\usepackage{microtype}      
\usepackage{xcolor}         

\usepackage{amssymb}            
\usepackage{mathtools}          
\usepackage{mathrsfs}           
\usepackage{graphicx}           
\usepackage{subcaption}         
\usepackage[space]{grffile}     
\usepackage{url}                
\usepackage{lipsum}             
\usepackage{xcolor}
\usepackage{colortbl}

\usepackage{thm-restate}

\definecolor{darkblue}{RGB}{0,0,139}
\definecolor{lightblue}{RGB}{173,216,230}
\definecolor{darkred}{RGB}{139,0,0}
\definecolor{lightred}{RGB}{255,204,204}

\newcommand{\poscolor}[1]{%
    \cellcolor{blue!\the\numexpr100*#1\relax!white} #1%
}

\newcommand{\negcolor}[1]{%
    \cellcolor{red!\the\numexpr-100*#1\relax!white} #1%
}

\usepackage{graphicx}
\usepackage{natbib}
\usepackage{mathtools}

\usepackage{amssymb}            
\usepackage{mathtools}          
\usepackage{mathrsfs}           
\mathtoolsset{showonlyrefs}     
\usepackage{graphicx}           
\usepackage{subcaption}         
\usepackage[space]{grffile}     
\usepackage{url}                

\usepackage{hyperref}
\usepackage{url}
\usepackage[utf8]{inputenc} 
\usepackage[T1]{fontenc}    
\usepackage{url}            
\usepackage{booktabs}       
\usepackage{amsfonts}       
\usepackage{nicefrac}       
\usepackage{microtype}      
\usepackage{xcolor}         
\usepackage{graphicx}
\usepackage{float}
\usepackage{amssymb}
\usepackage{mathtools}
\usepackage{amsthm}
\usepackage{multirow}
\usepackage{subcaption}
\usepackage{wrapfig}

\usepackage{microtype}
\usepackage{graphicx}
\usepackage{booktabs} 
\usepackage{amsmath}

\DeclareMathOperator*{\dooperator}{do}

\theoremstyle{plain}

\newtheorem{lemma}{Lemma}
\newtheorem{corollary}{Corollary}
\theoremstyle{definition}
\newtheorem{definition}{Definition}

\theoremstyle{remark}

\title{Breaking Habits: On the Role of the Advantage \\ Function in Learning Causal State Representations}

\author{%
  Miguel Suau \\
  Phaidra \\
  \texttt{miguel.suau@phaidra.ai} \\
}

\begin{document}

\maketitle

\begin{abstract}
  Recent work has shown that reinforcement learning agents can develop policies that exploit spurious correlations between rewards and observations. This phenomenon, known as policy confounding, arises because the agent’s policy influences both past and future observation variables, creating a feedback loop that can hinder the agent's ability to generalize beyond its usual trajectories. In this paper, we show that the advantage function, commonly used in policy gradient methods, not only reduces the variance of gradient estimates but also mitigates the effects of policy confounding. By adjusting action values relative to the state representation, the advantage function downweights state-action pairs that are more likely under the current policy, breaking spurious correlations and encouraging the agent to focus on causal factors. We provide both analytical and empirical evidence demonstrating that training with the advantage function leads to improved out-of-trajectory performance.
\end{abstract}

\section{Introduction}

Imagine a robot trained to perform two tasks: first, navigate from your office to the coffee machine, retrieve a cup of coffee, and return; second, go to the printer room, make copies, and return. There are two possible routes to the coffee machine: one through the printer room and another through a direct corridor. Since the corridor route is shorter, the robot typically avoids the printer room when fetching coffee. However, one day, the corridor is blocked and the robot takes the path through the printer room, unexpectedly returning with a copy of a paper titled \emph{Breaking Habits}.

Why did this happen? After repeatedly performing these tasks, the robot incorrectly associated the printer room with the need to make copies, turning this misassociation into a habit. While this habit works fine under normal conditions, it fails when the robot is forced to deviate from its usual path. This failure mode, referred to as \emph{out-of-trajectory generalization}, was explored by \citet{suau2024bad} in the context of reinforcement learning (RL). The authors showed that such issues arise because the agent’s policy introduces spurious correlations \citep{pearl2016causal} between rewards and observations, a phenomenon they termed \emph{policy confounding}.

\paragraph{Contributions}
In this paper, we observe that the advantage function, commonly used in many policy gradient methods, not only reduces the variance of gradient updates but also plays a crucial role in mitigating this issue. We demonstrate, both analytically and experimentally, that using the advantage function encourages the agent to learn representations that rely more on the true causal factors. It achieves this by scaling state-action pairs according to their probability under a given policy and state representation, effectively breaking the spurious correlations introduced by the agent’s policy. This encourages the agents to focus on causal factors, enabling better out-of-trajectory generalization.

To support the theoretical findings, experiments are conducted in three simple environments. The results indicate that agents trained with $Q$-values fail to generalize beyond their usual trajectories, whereas agents trained with the advantage function exhibit robustness to trajectory deviations. Furthermore, an analysis of the learned state representations reveals that the latter focus on causal factors rather than exploiting spurious correlations.

Finally, we examine the impact of implementation choices, such as batch size and advantage normalization. Through experiments, we show that these factors can significantly affect the learned state representations.

\section{Example: Key2Door}\label{sec:example}

We will use the following example throughout the paper to illustrate the ideas presented.

Figure \ref{fig:key2door} depicts a gridworld environment. The agent's objective is to collect a key placed at the beginning of the corridor and then open the door at the end. The agent's observation consists of its current location $L \in \{0, 1, .., 6\}$ and a binary variable $X \in \{0, 1\}$ indicating whether or not it has the key. At all states, the agent has two possible actions: moving left $A=0$, or moving right $A=1$. The agent receives a reward of $+1$ if it reaches the door after collecting the key, and $0$ if it reaches the door without the key. Additionally, there is a $-0.01$ penalty per timestep to incentivize the agent to take the most direct path. 
\begin{wrapfigure}{r}{0.6\textwidth} 
\begin{center} 
\vspace{-15pt} 
    \includegraphics[width=0.6\textwidth]{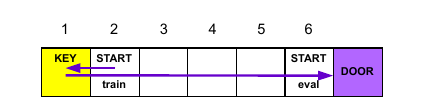} 
\end{center} 
\vspace{-10pt} 
\caption{Key2Door environment.} 
\vspace{-5pt} 
\label{fig:key2door} 
\end{wrapfigure} 

During training, the agent always starts the episode in location 2. This implies that once the agent has collected the key and is moving toward the door, it can disregard the key once it reaches location 3. The agent can learn to use its location to infer whether it has the key. However, this only works well if the agent consistently follows the optimal policy, as depicted by the purple arrow. If the agent moves randomly through the corridor, it cannot reliably determine whether the key has been collected. Moreover, an agent that relies solely on its location will fail when evaluated in the same task but starting from location 6. The agent will attempt to move right to open the door, even though it has not collected the key.

\section{Related work}
The term policy confounding was introduced by \citet{suau2024bad} to describe a phenomenon in which policies, by influencing both past and future observation variables, induce spurious correlations. These correlations can lead the agent to develop shortcuts, referred to as habits, that are effective only within the trajectories the agent typically follows but fail to generalize out-of-trajectory. Out-of-trajectory (OOT) generalization is a specific instance of the broader problem of out-of-distribution (OOD) generalization \citep{kirk2023survey}. Unlike OOD generalization, which focuses on adapting to environments with different rewards \citep{Taylor09ICML}, observations \citep{mandlekar2017adversarially,zhang2020invariant}, or transition dynamics \citep{higgins2017darla}, OOT generalization seeks to enable agents to generalize across alternative trajectories within the same environment. Empirical evidence of policy confounding has been reported in several studies \citep{machado2018revisiting, Song2020Observational, lan2023can, he2024model, weltevrede2024training}. Similarly, prior work has examined confounding in the context of imitation learning \citep{de2019causal, zhang2020causal, tennenholtz2021covariate, ding2023seeing}. However, unlike policy confounding, these works focus on cases where the agent does not actively contribute to the formation of spurious correlations.

The connection between the advantage function and causality has been explored in prior work. \citet{corcoll2022disentangling} use the advantage function to differentiate between agent-driven effects and external environmental factors, constructing a hierarchy of transformations that the agent can perform on its environment. \citet{pan2022direct} argue that the advantage function can be interpreted as a measure of an action’s causal effect on return and introduce a method for learning advantages directly from on-policy data without relying on a value function. This approach was later extended to support off-policy data \citep{pan2024skill}. While not explicitly framing the advantage function in causal terms, \citet{raileanu2021decoupling} describe it as a measure of the expected additional return gained by selecting a particular action over following the current policy. Their experiments suggest that the advantage function is less prone to overfitting to certain environment-specific idiosyncrasies.
 
Our work is strongly influenced by \citet{chung2021beyond}, who showed that the choice of baseline in policy gradient updates can sometimes lead to overly aggressive updates, causing what the authors term committal behavior, which may result in suboptimal policies. \citet{mei2022role} further analyzed the theoretical properties of state-value baselines (i.e., advantages), demonstrating that they moderate update aggressiveness and ensure convergence to the optimal policy. As we discuss in the following sections, these insights are closely connected to the findings presented in this work.
\section{Preliminaries}
We focus on problems where states are represented by a set of observation variables, or factors \citep{boutilier1999decision}. This representation is common in tasks that involve modeling policies and value functions using function approximators \citep{franccois2018introduction}. These observation variables typically describe features of the agent's state in the environment.
\begin{definition}[MDP] A Markov decision process (MDP) is a tuple $\langle \mathcal{S}, \mathcal{A}, T, R \rangle$, where $\mathcal{S}$ represents the set of states, $\mathcal{A}$ denotes the set of actions available to the agent, $T : \mathcal{S} \times \mathcal{A} \to \Delta (\mathcal{S})$ is the transition function, and $R : \mathcal{S} \times \mathcal{A} \to \mathbb{R}$ is the reward function. \end{definition}
\begin{definition}[FMDP] A Factored Markov decision process (FMDP) is an MDP where the set of states is defined by a set of observation variables, or factors, $F = {F^1, \dots, F^{|F|}}$. Each variable $F^i$ can take any value from its domain $\mathcal{F}^i$. Consequently, each state $s$ corresponds to a unique combination of values for the variables in $F$, i.e., $s = \langle f^1, \dots, f^{|F|} \rangle \in \times \mathcal{F} = \mathcal{S}$. \end{definition}

\subsection{State representations}
The agent's objective is to learn a policy $\pi: \mathcal{S} \to \Delta (\mathcal{A})$ that maximizes the expected discounted sum of rewards \citep{sutton2018reinforcement}. 
However, learning a policy that conditions on every observation variable might be impractical, especially in scenarios with a large number of variables. Fortunately, in many problems, not all variables are essential, and compact state representations can be found that are sufficient for solving the task at hand \citep{McCallum95PhD}. This is where function approximators, such as neural networks, come into play. When used to model policies and value functions, they learn to ignore certain observation variables in $F$ if those variables are deemed unnecessary for estimating values and optimal actions.

\begin{definition}[State representation] A state representation is a function $\Phi: \mathcal{S} \to \bar{\mathcal{S}}$, where $\mathcal{S} = \times \mathcal{F}$, $\bar{\mathcal{S}} = \times \bar{\mathcal{F}}$, and $\bar{F} \subseteq F$. \label{def:state_representation} \end{definition}

Intuitively, a state representation $\Phi(s_t)$ is a projection of a state $s \in \mathcal{S} = \times \mathcal{F}$ onto a lower-dimensional space $\bar{S} = \times \bar{\mathcal{F}}$, defined by a subset of its variables, $\bar{F} \subseteq F$. We use $\{s\}^\Phi = \{s' \in \mathcal{S}: \Phi(s') = \Phi(s)\}$ to denote the equivalence class of $s$ under $\Phi$.

In the Key2Door example a potential state representation could be $\Phi(s_t) = \langle l_t \rangle$ for all $s_t \in \mathcal{S}$. This representation retains only the agent's current location, ignoring all other variables in $F$. Therefore, all states sharing the same value for $L_t$ belong to the same equivalence class.

Not all state representations are sufficient to learn the optimal policy. Some representations may exclude variables that carry valuable information for solving the task. In the Key2Door example, knowing whether the key has been collected is crucial to selecting the next action.

\begin{definition}[Markov state representation] A state representation $\Phi(s_t)$ is said to be Markov if, for all $s_t, s_{t+1} \in \mathcal{S}$ and $a_t \in \mathcal{A}$, 

\begin{equation*} R(s_t , a_t) = R(\Phi(s_t), a_t) \quad \text{and} \quad \sum_{s'_{t+1} \in \{s_{t+1}\}^\Phi} T(s'_{t+1} \mid s_t, a_t) = P(\Phi(s_{t+1}) \mid \Phi(s_t), a_t), \end{equation*} 
\normalsize
where $R(\Phi(s_t), a_t)$ denotes the reward $R(s'_t, a_t)$ for any $s'_t \in \{s_t\}^\Phi$. \label{def:markov_state} \end{definition}


\subsection{Policy confounding}
The phenomenon of policy confounding plays a critical role in the search for simpler state representations. The policy can mislead the function approximator into forming simpler state representations that rely on spurious correlations, such as having the key when being at location 6, rather than true causal factors.

\begin{definition}[Policy Confounding] A state representation $\Phi : \mathcal{S} \to \bar{\mathcal{S}}$ is said to be confounded by a policy $\pi$ if, for some $s_t, s_{t+1} \in \mathcal{S}$ and $a_t \in \mathcal{A}$, 

\begin{equation*} 
R^\pi(\Phi(s_t), a_t) \neq R^\pi(\dooperator(\Phi(s_t)), a_t)
\end{equation*}
or
\begin{equation*}
P^\pi(\Phi(s_{t+1}) \mid \Phi(s_t), a_t) \neq P^\pi(\Phi(s_{t+1}) \mid \dooperator(\Phi(s_t)), a_t). 
\end{equation*}
\normalsize
where $R^\pi(\Phi^\pi(s_t), a_t)$ is the reward $R(s'_t, a_t)$ at any $s'_t \in \{s_t\}^\Phi_\pi$, with $\{s\}^\Phi_\pi =\{s' \in S^\pi: \Phi^\pi(s') = \Phi^\pi(s)\} $, and $P^\pi$ is probability under $\pi$.
\end{definition}

The operator $\dooperator(\cdot)$ is known as the do-operator, and it is used to represent physical interventions in a system \citep{pearl2016causal}. These interventions are meant to distinguish cause-effect relations from mere statistical associations. In our case, $\dooperator(\Phi(s_t))$ means setting the variables forming the state representation $\Phi(s_t)$ to a particular value and considering all possible states in the equivalence class, $s'_t \in \{s_t\}^\Phi$, (i.e., all states that share the same value for the observation variables that are selected by $\Phi$; independently of whether these are visited by the policy being followed). For instance, in the Key2Door example, $R^{\pi^*}(L = 6) = +1$ when following the optimal policy $\pi^*$ (purple path in Figure \ref{fig:key2door}) since we know that the agent has the key, while $R^{\pi^*}(\dooperator(L = 6)) = \pm 1$ since the agent may or may not have the key.

Strictly speaking, the only representation that satisfies the above conditions is the Markov state representation. This representation is invariant to the agent's policy, making it robust to any deviations in the trajectory. In the following, we refer to Markov state representations as \textit{Causal State Representations}, as they necessarily include all the causal factors governing rewards and transitions.

\subsection{Policy gradient and advantage function}\label{sec:policy_gradient}

\citet{suau2024bad} demonstrated that the phenomenon of policy confounding is particularly problematic when training agents with on-policy methods. This is because, when updating the policy, on-policy methods rely solely on trajectories collected using the current policy. This contrasts with off-policy methods, where the agent is trained on trajectories generated by multiple policies, thus broadening the trajectory distribution and reducing the risk of the agent picking up on spurious correlations present in specific trajectories.

A popular family of on-policy methods includes those that directly optimize the policy by following the gradient of the expected return with respect to the policy parameters, \(\theta\). The policy gradient theorem \citep{marbach1999simulation, sutton1999policy} formalizes this as:

\begin{equation}
\nabla_\theta J(\pi_\theta) = \mathbb{E}_{s_t, a_t \sim \pi_\theta} \left[ \frac{\nabla_\theta \pi_\theta(a_t \mid s_t)}{\pi_\theta(a_t \mid s_t)} Q^\pi(s_t, a_t) \right],
\end{equation}
\normalsize
where \(Q^\pi(s_t, a_t)\) is the action-value function under policy \(\pi\).

Following this gradient increases the likelihood of sampling actions that lead to high returns while reducing the probability of actions leading to lower returns. However,
in practice, computing the exact gradient is infeasible. Instead, we approximate the gradient using sample estimates from trajectories collected by the policy, which introduces high variance.

A common strategy to reduce variance is to subtract the state value function \(V^\pi(s_t)\) from the \(Q\)-value \citep{Baird94ICoNN, greensmith2001variance}, leading to the definition of the advantage function:

\begin{equation}
A^\pi(s_t, a_t) := Q^\pi(s_t, a_t) - V^\pi(s_t).
\end{equation}
\normalsize
This adjustment does not introduce bias, as the value function is independent of the action, but it significantly reduces the variance of the gradient estimation.

\section{The role of the advantage in learning causal state representations}\label{sec:advantage_role}

In this section, we demonstrate that, beyond its well-known role in variance reduction, the advantage function plays a crucial role in mitigating policy confounding. Specifically, it helps break any spurious correlations that the policy may introduce, thus preventing the agent from forming habits that fail to generalize beyond its typical trajectories. Proofs for the theoretical
results in this section can be found in Appendix \ref{ap:proofs}.

\subsection{Value and advantage function under a state representation}\label{sec:advantage_phi}
Computing the advantage usually involves fitting a model to approximate the value function. Conceptually, the value function can be decomposed into a state representation function $\Phi(s)$, which projects the states into a lower-dimensional space defined by the subset of variables relevant to the task, and a function $V^\pi(\Phi(s))$ that maps the state representation to a scalar value. 
%

\begin{definition}[$V^\pi$ and $Q^\pi$ under $\Phi$]
Let $\Phi : \mathcal{S} \to \bar{\mathcal{S}}$ be a state representation (Definition~\ref{def:state_representation}) that induces an equivalence class $\{s\}^\Phi := \{s' \in \mathcal{S} : \Phi(s') = \Phi(s)\}$. Assume the Markov chain under policy $\pi$ admits a stationary distribution $d^\pi$ over $\mathcal{S}$. Then the state value $V^\pi: \bar{\mathcal{S}} \to \mathbb{R}$ and state-action value $Q^\pi: \bar{\mathcal{S}} \times \mathcal{A} \to \mathbb{R}$ under $\Phi$ are defined as:

\[
    V^\pi(\Phi(s_t)) := \sum_{s_t' \in \{s_t\}^\Phi} P^\pi(s_t' \mid \Phi(s_t)) \, V^\pi(s_t') 
\]
and
\[
    Q^\pi(\Phi(s_t), a_t) := \sum_{s_t' \in \{s_t\}^\Phi} P^\pi(s_t' \mid \Phi(s_t)) \, Q^\pi(s_t', a_t),
\]
\normalsize
where the distribution $P^\pi(s_t' \mid \Phi(s_t))$ is defined as:

\[
    P^\pi(s_t' \mid \Phi(s_t)) := \frac{d^\pi(s_t')}{\sum_{s \in \{s_t\}^\Phi} d^\pi(s)} \quad \text{for all } s_t' \in \{s_t\}^\Phi.
\]
\end{definition}

\begin{lemma}
Let $\Phi$ be a Causal (Markov) State Representation (Definition \ref{def:markov_state}), then
 
\begin{equation}
    Q^\pi(s_t, a_t) = Q^\pi(s'_t, a_t) = Q^\pi(\Phi(s_t), a_t)  \qquad \text{and} \qquad V^\pi(s_t) = V^\pi(s'_t) = V^\pi(\Phi(s_t))
\end{equation}   
\normalsize
for all $a \in \mathcal{A}$, $s \in \mathcal{S}$, and, $s' \in \{s\}^\Phi$.
\label{lm:v_q_phi}
\end{lemma}
\begin{definition}[$A_\Phi^\pi$ under $\Phi$]\label{def:advantage_phi}
    Given a state representation $\Phi : \mathcal{S} \to \mathcal{\bar{S}}$, the advantage function $A_\Phi^\pi: \mathcal{S} \to \mathbb{R}$ under $\Phi$ is defined as:
    
    \begin{equation}
        A^\pi_\Phi(s_t, a_t) := Q^\pi(s_t, a_t) - V^\pi(\Phi(s_t)).
    \end{equation}
    \normalsize
\end{definition}
Despite the agent's state representation being $\Phi(s_t)$, Definition \ref{def:advantage_phi} shows that the advantage function $A_\Phi^\pi$ still depends on the full state $s_t$. This is because $Q^\pi(s_t, a_t)$ is typically estimated directly for each state $s_t$ and action $a_t$ from Monte Carlo rollouts, making it independent of $\Phi(s_t)$. In contrast,  $V^\pi(\Phi(s_t))$ averages over the value of all states equivalent under $\Phi$, effectively marginalizing out the variables not preserved by $\Phi$. As a result, $A_\Phi^\pi$ measures the deviation of a particular state-action pair from the expected value across the equivalence class $\{s_t\}^\Phi$.

\subsection{The scaling effect of the advantage function}
\begin{restatable}{theorem}{theoremone}
Let $\Phi$ be an arbitrary state representation. The advantage function can be expressed as:

    \begin{equation}
    A^\pi_\Phi(s_t, a_t) = (1 - P^\pi(s_t, a_t \mid \Phi(s_t)))( Q^\pi(s_t, a_t) - \tilde{Q}^\pi(\neg \langle s_t, a_t\rangle))
    \end{equation}
\normalsize
where 

    \begin{equation}
        \tilde{Q}_\Phi^\pi(\neg \langle s_t, a_t\rangle) = \frac{\sum_{ s'_t, a'_t  \neq  s_t, a_t } P^\pi(s'_t, a'_t \mid \Phi(s_t)) Q^\pi(s'_t, a'_t)}{\sum_{ s'_t, a'_t  \neq  s_t, a_t } P^\pi(s'_t, a'_t \mid \Phi(s_t))}
    \end{equation}
\label{thm:advantage_phi}
\end{restatable}
This result highlights that the advantage function is proportional to the complement of the joint probability of a specific state-action pair conditioned on the state representation, i.e., $1 - P^\pi(s_t, a_t \mid \Phi(s_t))$. This ensures that the advantage of less probable state-action pairs is larger in magnitude, and helps offset the overrepresentation of more frequent pairs under $\pi$. 

This scaling effect is crucial for mitigating policy confounding as it breaks any spurious correlations induced by the policy. By amplifying rare but informative samples, the advantage function encourages the agent to attend to underlying causal factors rather than simply reinforcing habits. Intuitively, it reweights each sample according to its causal content, preventing high-probability pairs from dominating the gradient updates during training.

Moreover, the difference \( Q^\pi(s_t, a_t) - \tilde{Q}^\pi_\Phi(\neg \langle s_t, a_t \rangle) \) reflects how much better (or worse) a given pair is relative to the alternative pairs. Notably, \( \tilde{Q}^\pi_\Phi(\neg \langle s_t, a_t \rangle) \) is a weighted average over state-action pairs in the same equivalence class, excluding $\langle s_t, a_t \rangle$, and since it does not depend on $P^\pi(s_t, a_t \mid \Phi(s_t))$, it provides a stable, representation-conditional baseline for comparison. 

\begin{corollary}
Let $\Phi$ be a Causal (Markov) State Representation. The advantage function can be expressed as:

    \begin{equation}
    A^\pi_\Phi(s_t, a_t) = (1 - \pi(a_t \mid \Phi(s_t)))( Q^\pi(s_t, a_t) - \tilde{Q}^\pi(s_t, \neg a_t))
    \end{equation}
\normalsize
where 

    \begin{equation}
        \tilde{Q}^\pi(s_t, \neg a_t) = \frac{\sum_{a'_t  \neq  a_t} \pi(a'_t \mid \Phi(s_t)) Q^\pi(s_t, a'_t)}{\sum_{a'_t  \neq  a_t} \pi(a'_t \mid \Phi(s_t))}
    \end{equation}
\end{corollary}
The above result follows directly from Theorem \ref{thm:advantage_phi} and Lemma \ref{lm:v_q_phi}. Note that when the state representation $\Phi$ is causal, there is no need to scale the sample gradients according to the likelihood of the states given $\Phi$, as $\Phi$ already accounts for the true factors governing rewards and transitions.


\subsection{Impact on policy gradients}

None of the above would matter if the updates strictly followed the exact policy gradient, as the advantage function does not alter the gradient \citep{sutton2018reinforcement}. However, as discussed in Section \ref{sec:policy_gradient}, policy gradients are typically based on sample estimates (i.e., stochastic gradients).

As a result, depending on the agent's policy and the batch size, it is often the case that for a given $\Phi(s_t)$, all samples in a training batch correspond to a single state $s_t$, rather than including all states $s'_t \in \{s\}^\Phi$. Consequently, when using $Q$-values instead of advantages $A^\pi_\Phi$, this can lead to overly aggressive gradient updates \citep{chung2021beyond}, causing the agent to develop habits based on spurious correlations. This issue is further exacerbated as the policy reinforces the likelihood of frequently observed state-action pairs, creating a vicious circle \citep{mei2022role}. In the extreme, when the policy becomes deterministic, certain state-action pairs may have their probabilities driven to zero, requiring an infeasibly large number of samples to correct for the overestimation of the state-action pairs that are actually visited.

In contrast, the advantage function $A^\pi_\Phi$ scales the gradients by adjusting their magnitude based on the probability of each state-action pair (Theorem \ref{thm:advantage_phi}). This moderates the aggressiveness of gradients for frequently sampled state-action pairs while boosting gradients for less common ones, thereby helping the agent break bad habits.

\subsection{Practical considerations}\label{sec:practical_considerations}
Given the points above, an alternative to using advantages could be to employ large batch sizes. However, the batch size may need to be significantly increased depending on the environment for this approach to be effective. Moreover, we find no compelling reason to avoid using advantages.

It is also important to note that the results reported by \citet{suau2024bad} indicate that PPO struggles with out-of-trajectory generalization, despite the default setting utilizing advantages. We hypothesize that this issue arises because, in most PPO implementations, the advantages in each training batch are normalized before the network update. This normalization removes the scaling effect of the advantage function on the gradients. Both of these aspects are further analyzed in the experiments section.

\subsection{Example: Key2Door}\label{sec:advatange_example}

To illustrate the insights discussed in the previous sections, we revisit the Key2Door example (Section \ref{sec:example}). Suppose the agent's state representation consists solely of the location variable, $\Phi(s) = l$. Table \ref{table:tab1} shows the $Q$-values and corresponding advantages $A^\pi_\Phi$ when the agent is at location 6, both with and without the key, under five different policies. These policies differ in the probability of selecting the optimal action: $0.5$, $0.6$, $0.7$, $0.8$, and $0.9$. The last two columns indicate the probability that the agent has, $P^\pi(X=1 \mid L=6)$, or has not, $P^\pi(X=0\mid L=6)$, the key at location 6.\footnote{$Q$-values and advantages are computed using value iteration. $P^\pi(X \mid L=6)$ is estimated by running the policy over multiple episodes and calculating the frequency of each state.
}

\begin{table}[!ht]
\centering
\caption{$Q$-values, advantages, and probabilities of key and no key when the agent is at location 6 under five different policies. The agent's state representation consists solely of the location variable.}
\resizebox{\textwidth}{!}{%
\begin{tabular}{lcccccccccc}
\toprule
& \multicolumn{4}{c}{\textbf{$Q$-value}} 
& \multicolumn{4}{c}{\textbf{Advantage $\Phi(s) = l$}} 
& \multicolumn{2}{c}{$P^\pi(X \mid L=6)$} \\ 
\cmidrule(r){2-5} \cmidrule(r){6-9} \cmidrule(r){10-11}
& \multicolumn{2}{c}{\textbf{No Key}} & \multicolumn{2}{c}{\textbf{Key}} 
& \multicolumn{2}{c}{\textbf{No Key}} & \multicolumn{2}{c}{\textbf{Key}} 
& \textbf{No Key} & \textbf{Key} \\ 
\cmidrule(r){2-3} \cmidrule(r){4-5} \cmidrule(r){6-7} \cmidrule(r){8-9} 
& \textbf{Left} & \textbf{Right} & \textbf{Left} & \textbf{Right}
& \textbf{Left} & \textbf{Right} & \textbf{Left} & \textbf{Right} 
&  &  \\ 
\midrule
$\pi(a^* \mid s) = 0.5$  & 0.038 & 0     & 0.662 & 1  & -0.657 & -0.695 & -0.033 & 0.305 & 0.168 & 0.832 \\
$\pi(a^* \mid s) = 0.6$  & 0.274 & 0     & 0.839 & 1  & -0.608 & -0.882 & -0.043 & 0.118 & 0.069 & 0.931 \\
$\pi(a^* \mid s) = 0.7$  & 0.504 & 0     & 0.905 & 1  & -0.456 & -0.960 & -0.055 & 0.040 & 0.018 & 0.982 \\
$\pi(a^* \mid s) = 0.8$  & 0.664 & 0     & 0.934 & 1  & -0.321 & -0.985 & -0.051 & 0.015 & 0.003 & 0.997 \\
$\pi(a^* \mid s) = 0.9$  & 0.759 & 0     & 0.950 & 1  & -0.235 & -0.994 & -0.044 & 0.006 & 0.000 & 0.999 \\
\bottomrule
\end{tabular}
}
\vspace{-5pt}
\label{table:tab1}
\end{table}

As revealed by Theorem \ref{thm:advantage_phi}, the magnitudes of the advantages depend on the joint probability of visiting a specific state-action pair. In particular, the advantage decreases (increases) in magnitude as the probability $P^\pi(X, A \mid L =6)$ increases (decreases). For example, the advantage of moving right \emph{with} the key is much larger when the probability of having the key is $0.832$ compared to when it is $0.999$ ($0.305$ vs. $0.006$). This difference is amplified by the increasing probability of taking the right action ($0.5$ vs. $0.9$). In contrast, the corresponding $Q$-values remain constant at $1$.

Similarly, the advantage of moving right \emph{without} the key is smaller (in magnitude) when the probability of not having the key is $0.168$ than when the probability is $0.000$ ($-0.695$ vs. $-0.994$), even though the $Q$-value remains $0$.\footnote{Note that the magnitude of the advantage of moving left \emph{with} the key first increases and then decreases across rows. This occurs because, although the probability of having the key increases, the probability of moving left decreases.}

The above is key because, as the policy improves, the probability of visiting state-action pairs such $\langle L=6, X=1, A=1 \rangle$ increases, and hence the training batches start to fill with many such pairs. Meanwhile, less frequent but informative pairs, such as $\langle L = 6, X = 0, A = 1 \rangle$, become underrepresented. As a result, training on raw $Q$-values can lead the agent to disregard the key variable $X$, effectively exploiting a spurious correlation between location 6 and having the key. This occurs because updates to the policy treat all state-action pairs equally, regardless of how often they appear in the training data. Notably, even under a fully random policy (top row of Table~\ref{table:tab1}), the probability of having the key at location 6 is already $0.832$,  meaning the agent often receives a reward of $1$ simply by moving right at location 6, without explicitly reasoning about whether it has the key.

In contrast, agents trained with the advantage function are less likely to ignore the key variable. When the policy is random, the advantage of moving right \emph{without} the key has a large negative value ($-0.695$), whereas the advantage of moving right \emph{with} the key is just $0.305$, despite the corresponding Q-values being $0$ and $1$ respectively. These differences in magnitude help counteract the overrepresentation of states such as $\langle L=6, X=1, A=1 \rangle$, and ensure that the agent does not develop the habit of moving solely based on location.

\section{Experiments}
The experiments aim to verify whether the insights discussed in the previous section hold in practice. Specifically, we seek to demonstrate that training on the advantage function, rather than raw $Q$-values, helps agents develop state representations that better capture causal factors and thus generalize out-of-trajectory. To test this, we conduct experiments on three gridworld environments: the Key2Door environment described in Section \ref{sec:example}, as well as the Frozen T-Maze and Diversion environments introduced by \citet{suau2024bad}. 

We evaluate the agent's performance in both the training environments and in modified versions, referred to as the evaluation environments, where, like in the Key2Door environment, the agent is forced to deviate from its usual trajectory. Furthermore, we analyze the effects of advantage normalization and batch size, which, as discussed in section \ref{sec:practical_considerations}, can influence out-of-trajectory generalization. Finally, we inspect the state representations learned by the agents by measuring the KL divergence of the policies between various state observations. Details about the T-Maze and Diversion environments are provided in Appendix \ref{ap:environments}.

\subsection{Experimental setup}\label{sec:experimental_setup}
Agents are trained using two different on-policy policy-gradient methods, REINFORCE \citep{Williams92ML} and PPO \citep{schulman2017proximal} to maximize either the advantage or the $Q$-value. We implement policies and value functions as feedforward neural networks and use a stack of past observations as input in environments that require memory. The results are averaged over 10 random seeds. We report the average return as a function of the number of training steps. The shaded areas show the standard error of the mean. Training is interleaved with periodic evaluations in the original environments and their variants.  Further details about the experimental setup are provided in Appendix \ref{ap:experimental_setup}.

\subsection{Results}
\begin{figure}[t]
    \centering
    \begin{subfigure}[t]{0.32\textwidth}
        \centering
        \includegraphics[width=\textwidth]{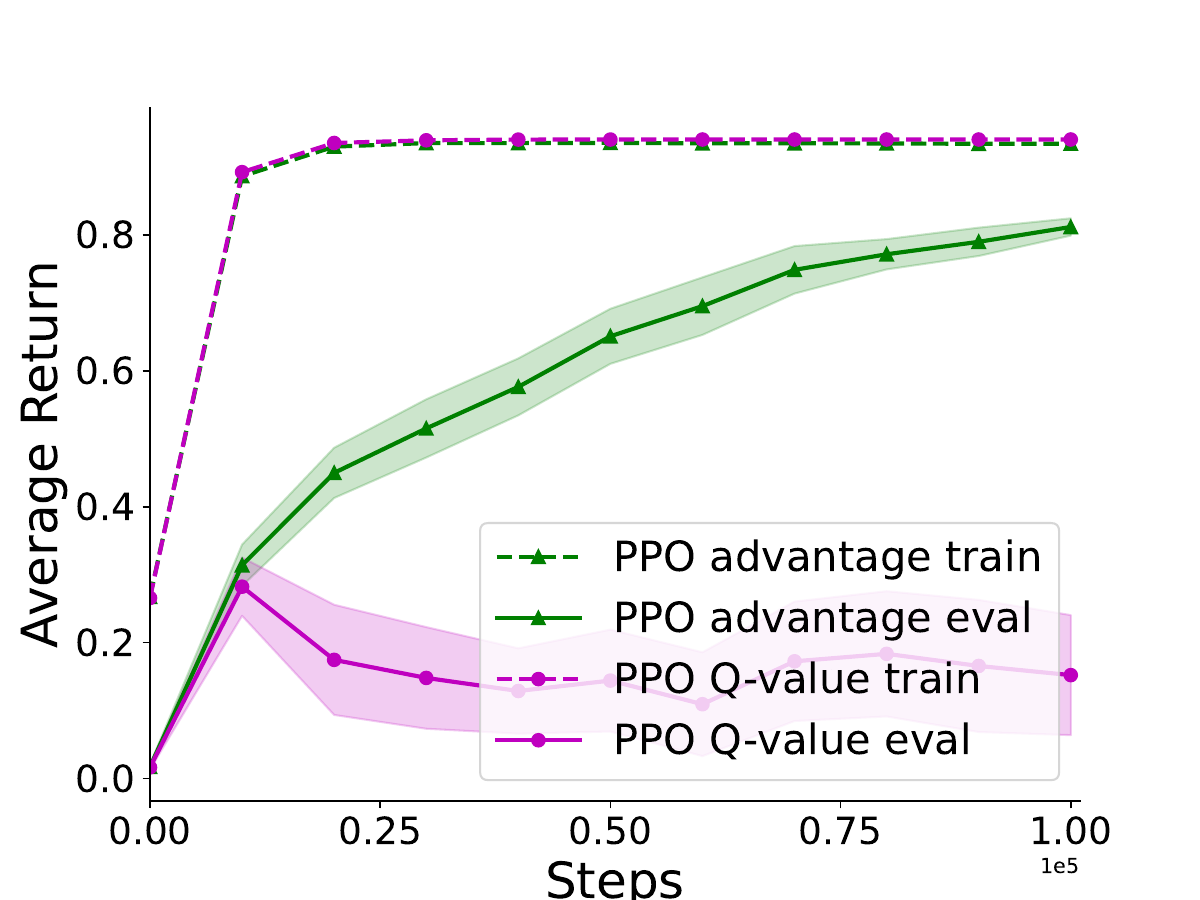}
    \end{subfigure}
    \hfill
    \begin{subfigure}[t]{0.32\textwidth}
        \centering
        \includegraphics[width=\textwidth]{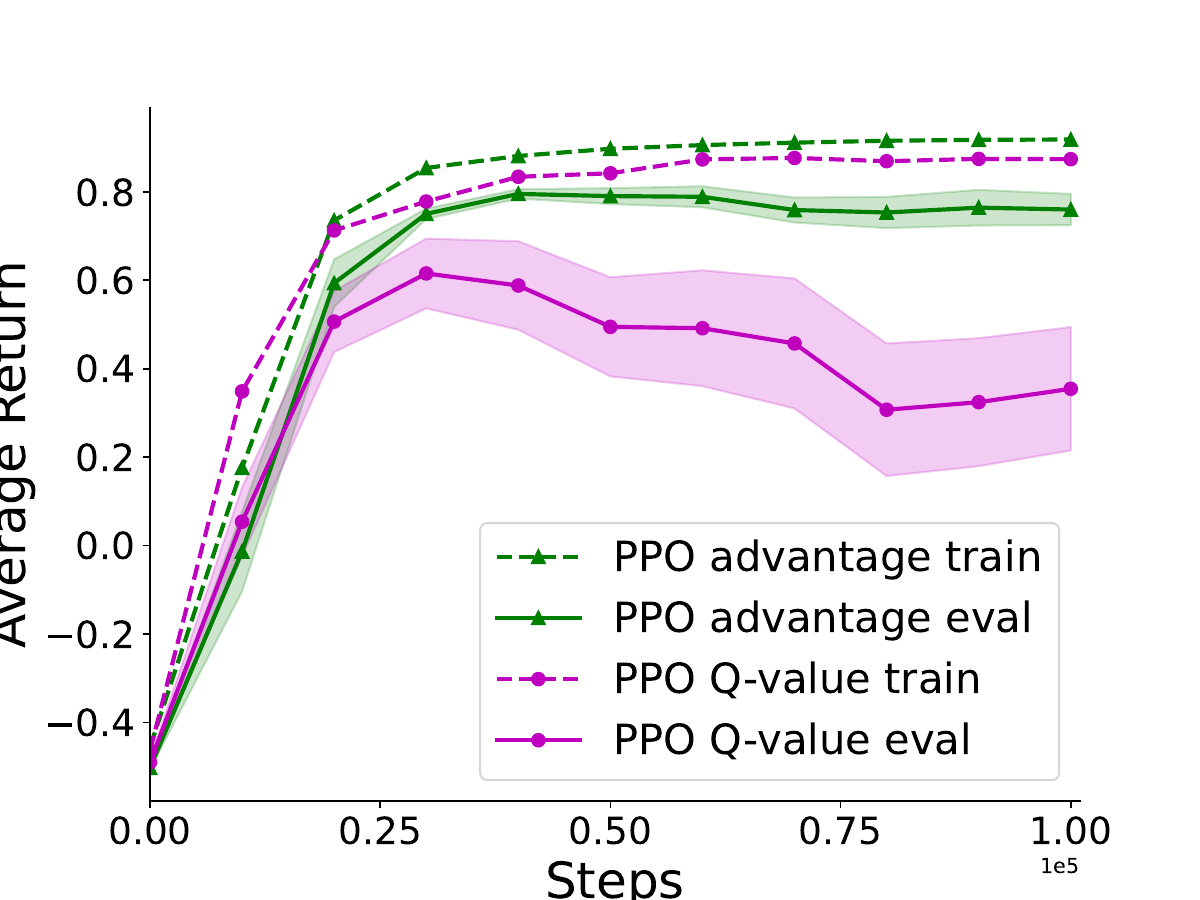}
    \end{subfigure}
    \hfill
    \begin{subfigure}[t]{0.32\textwidth}
        \centering
        \includegraphics[width=\textwidth]{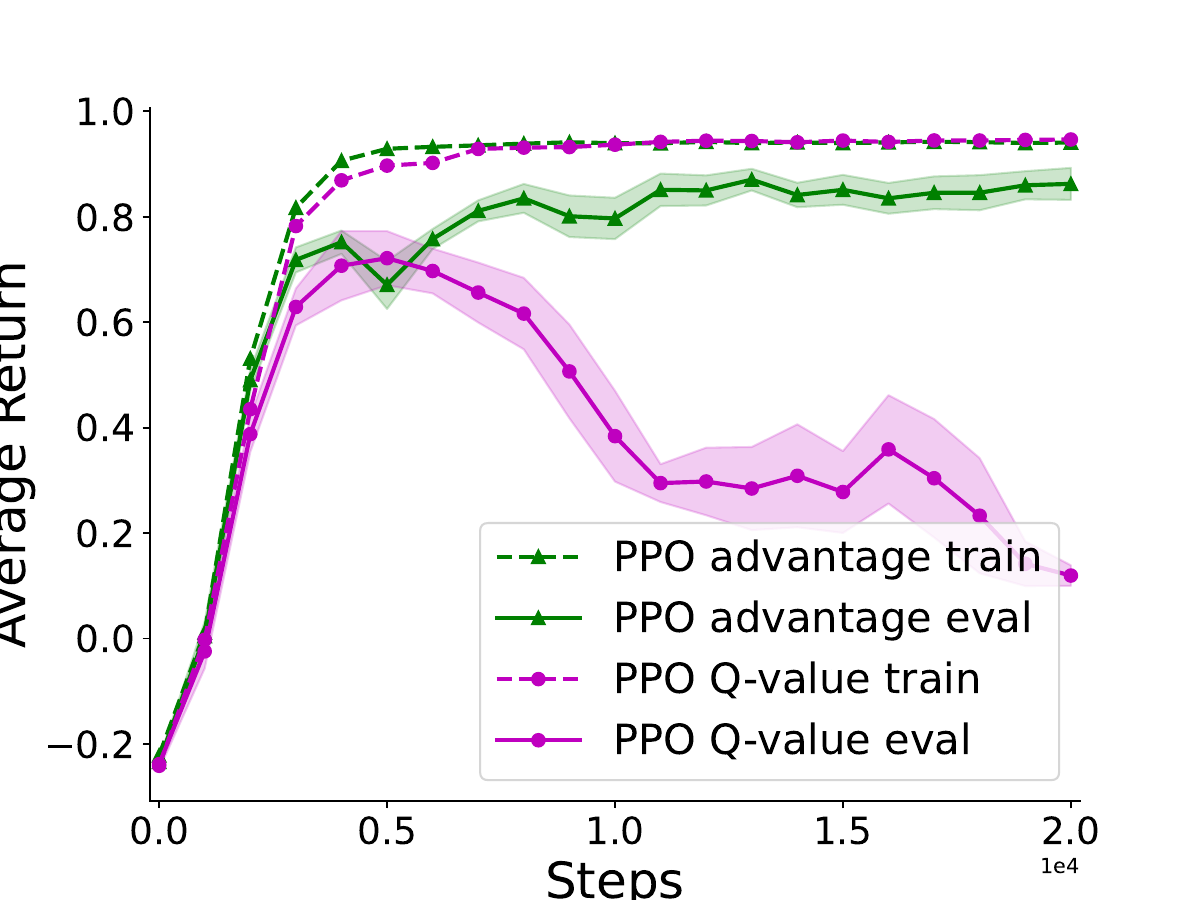}
    \end{subfigure}
    \hfill
    \caption{Performance of PPO using Q-values and Advantages in both the training and evaluation variants of the Key2Door (first plot), Frozen T-Maze (second plot), and Diversion (third plot).
}
\vspace{-10pt}
    \label{fig:qvsa}
\end{figure}

\begin{wrapfigure}{r}{0.65\textwidth} 
\vspace{-25pt}
\begin{center} 
    \begin{subfigure}[t]{0.325\textwidth}
        \centering
        \includegraphics[width=\textwidth]{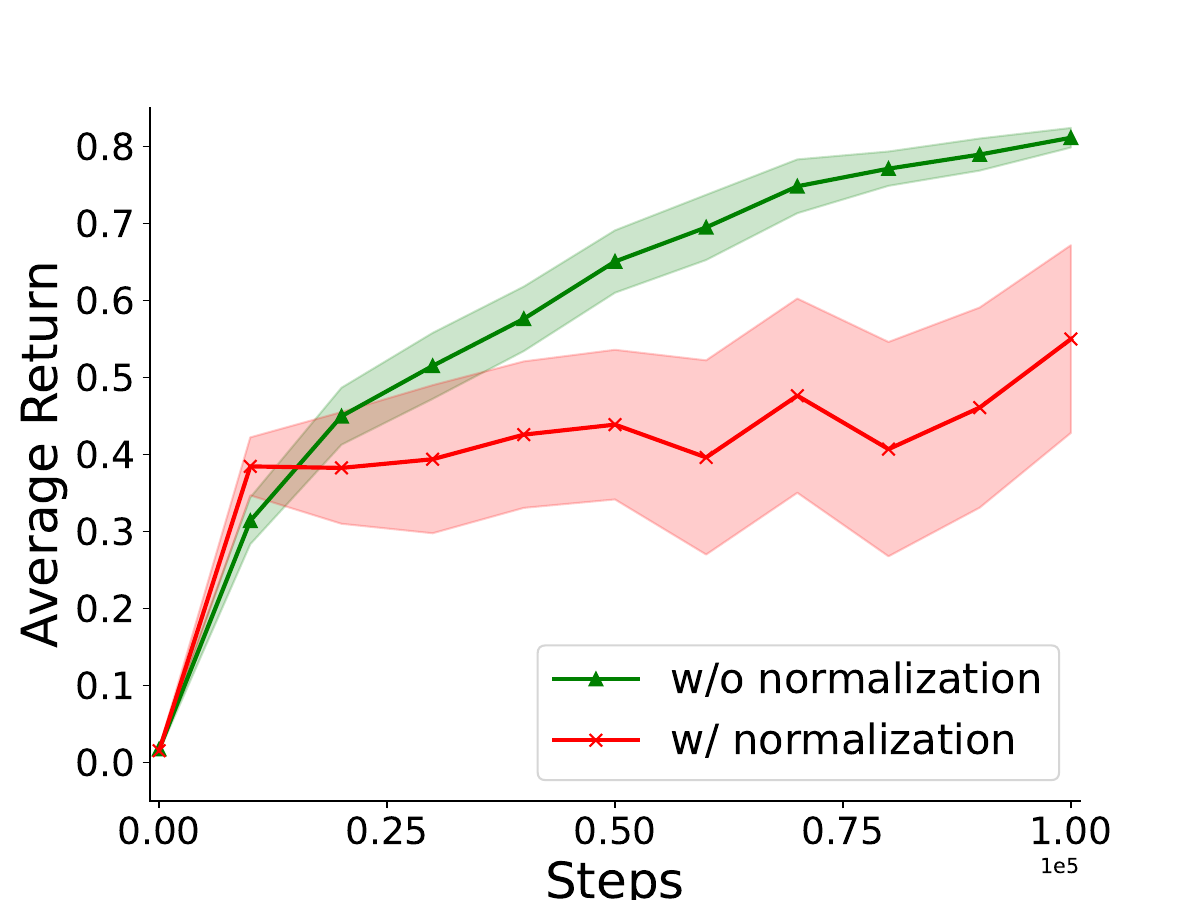}
    \end{subfigure}
    \hspace{-5pt}
    \begin{subfigure}[t]{0.325\textwidth}
        \centering
        \includegraphics[width=\textwidth]{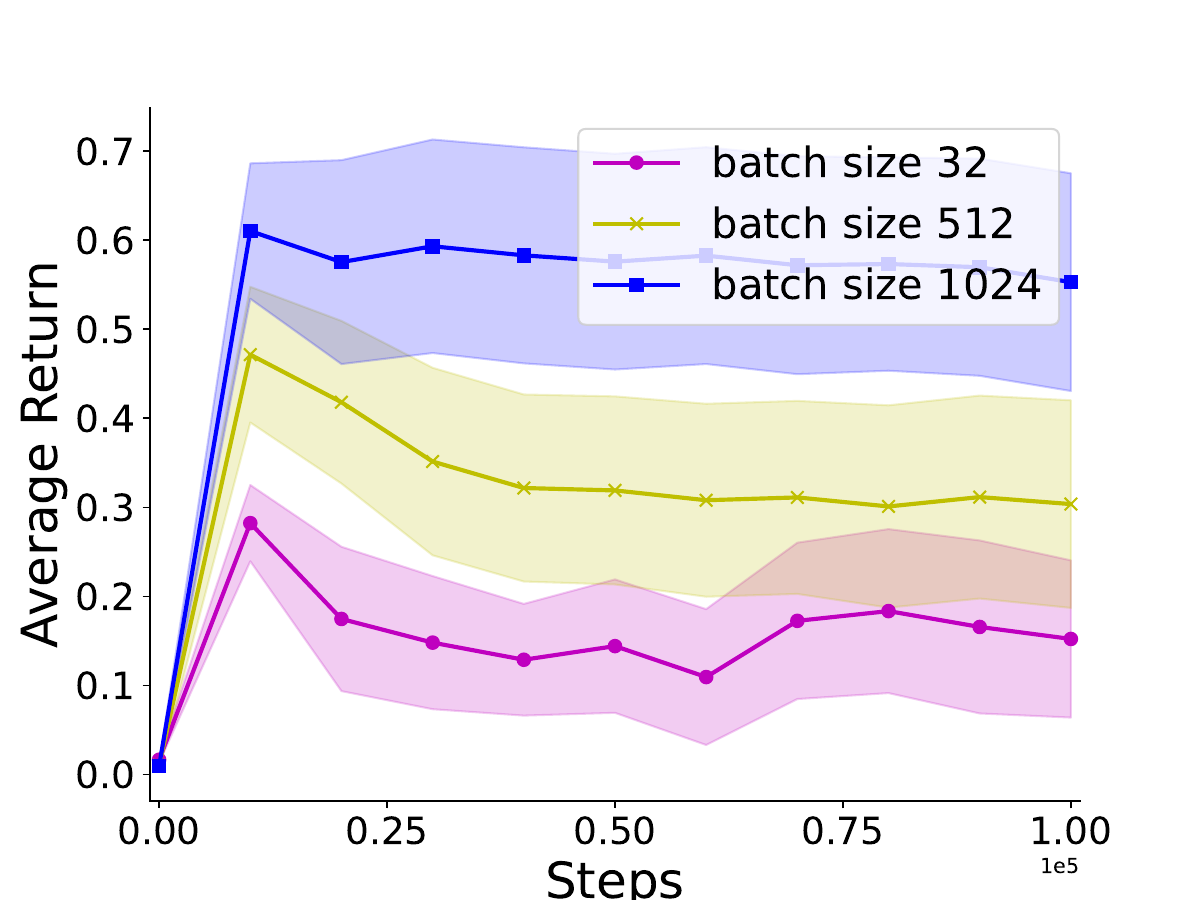}
    \end{subfigure}
\end{center} 
\caption{\textbf{Left:} Performance of PPO with and without advantage normalization in the Key2Door evaluation environment. \textbf{Right:} Performance of PPO with different batch sizes in the Key2Door evaluation environment.} 
\vspace{-5pt}
\label{fig:normalize_batchsize} 
\end{wrapfigure}
Figure \ref{fig:qvsa} shows the performance curves in all three environments for PPO. Agents trained using the advantage function (green) perform well in both the training and evaluation environments. In contrast, agents trained on the $Q$-value (magenta) perform poorly in the evaluation environments. Similar results using REINFORCE are reported in Appendix \ref{ap:reinforce}.

The plot on the left of Figure \ref{fig:normalize_batchsize} reveals how, as discussed in Section \ref{sec:practical_considerations}, normalizing the advantages removes their scaling effect and results in agents being unable to perform well on the evaluation environment. The plot on the right, on the other hand, shows how the performance of policies trained on the $Q$ value improves as we increase the batch size, suggesting that the problem of state-action pair imbalance can sometimes be partly mitigated by using larger batch sizes. Results for the other two environments are provided in Appendices \ref{ap:advantage_normalization} and \ref{ap:batch_size}.

The heatmaps in Figure \ref{fig:heatmaps} show the KL divergence of action probabilities between the agent having the key and not having the key at each of the six locations, measured at different training steps (top: 10k steps, middle: 50k steps, bottom: 100k steps). 

\begin{wrapfigure}{r}{0.65\textwidth} 
\vspace{-10pt}
\begin{center} 
    \begin{subfigure}[t]{0.325\textwidth}
        \centering
        \includegraphics[width=\textwidth]{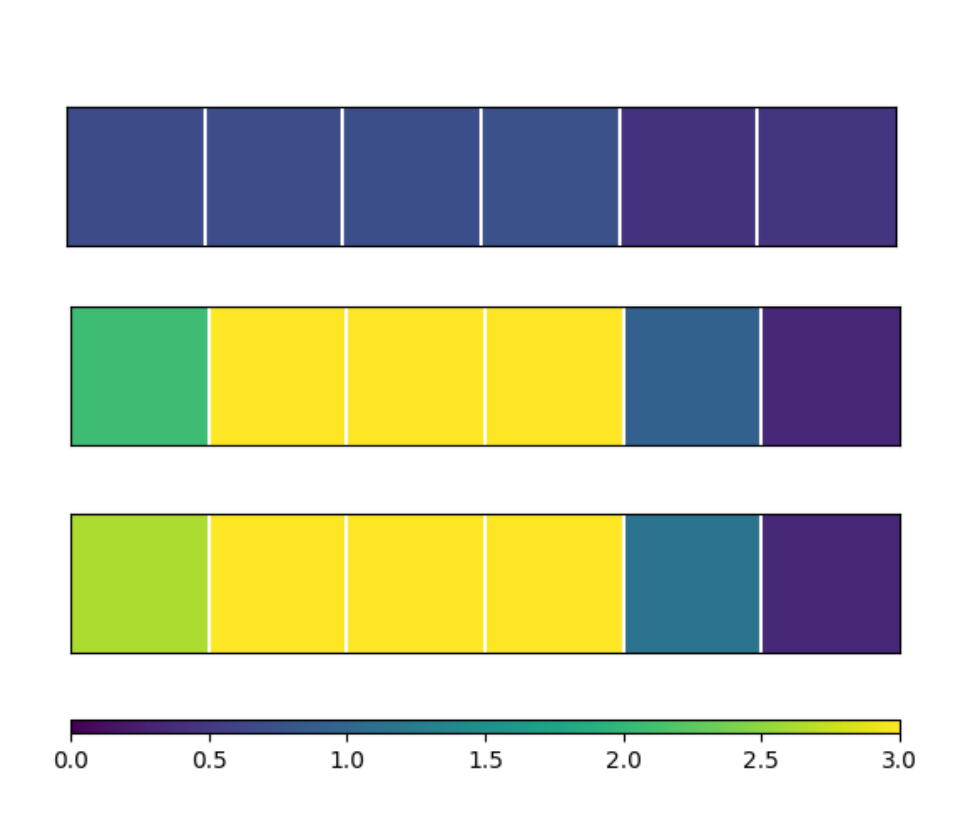}
    \end{subfigure}
    \hspace{-5pt}
    \begin{subfigure}[t]{0.325\textwidth}
        \centering
        \includegraphics[width=\textwidth]{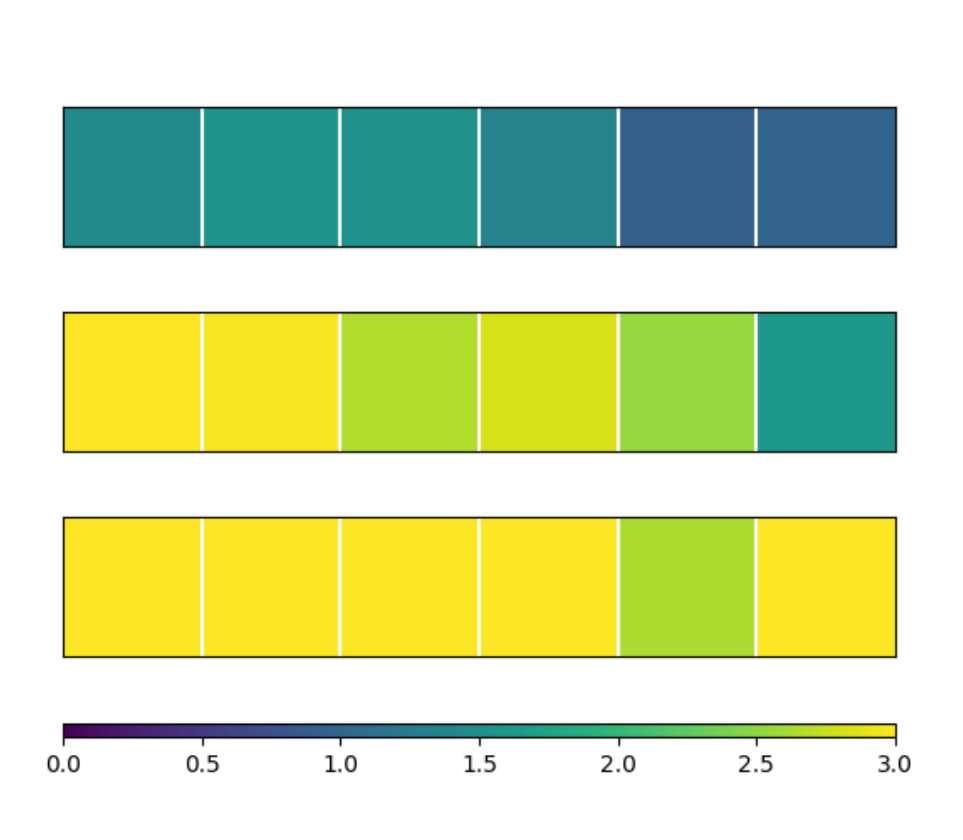}
    \end{subfigure}
\end{center} 
\vspace{-5pt} 
\caption{KL divergence of action probabilities with and without the key, measured at different training steps (top 10k steps, middle 50k steps, and bottom 100k steps) for agents trained on the $Q$-value (left) and the Advantage function (right).} 
\label{fig:heatmaps} 
\end{wrapfigure}The heatmap on the left corresponds to agents trained on the $Q$-value function, while the one on the right corresponds to agents trained on the advantage. Higher KL divergence indicates that the policy places more importance on the key feature. The results reveal that agents trained to optimize $Q$-value may learn policies that do not necessarily rely on the causal factors due to policy confounding, while policies trained on advantage function may be less affected by this phenomenon. Similar results for the Frozen T-Maze and Diversion environments are provided in Appendix \ref{ap:heatmaps}.

\section{Limitations}
Our analysis assumes discrete state and action spaces, although the results could be extended to continuous domains. In Section \ref{sec:advantage_phi}, when defining the advantage function under a given state representation $\Phi$ (Definition \ref{def:advantage_phi}), we assume that the value function $Q(s_t, a_t)$ is estimated with respect to the full state, while $V^\pi(\Phi(s))$ is defined for the specific state represantion. This assumption is justified in settings where $Q^\pi$ is at least partly estimated from Monte Carlo rollouts, as is common in many actor-critic methods. However, if $Q$ were instead learned entirely via function approximation, then it too would be subject to $\Phi$, and Theorem \ref{thm:advantage_phi} would no longer hold in its current form.

The experiments were conducted in the same three environments introduced by \citet{suau2024bad}, which were specifically designed to expose the phenomenon of policy confounding. While these environments are deliberately simple to facilitate analysis, this simplicity limits the generalizability of our findings. As such, we draw no conclusions about the effectiveness of the advantage function for learning causal representations in more complex or high-dimensional domains. Addressing this question would require further empirical investigation and is left for future work.

Finally, throughout the paper, we have taken care not to make strong claims about the effectiveness of the advantage function in learning causal state representations. Rather, we argue that the advantage function can mitigate policy confounding and support the formation of more causally grounded representations. However, its use provides no guarantees that the resulting representations will be truly causal.

\section{Conclusion}
In this paper, we analyzed the role of the advantage function in helping agents learn causal state representations. We showed that the advantage function scales the gradients by the complement of the probability of the corresponding state-action pair. This increases the magnitude of the gradients for state-action pairs that are less likely under the current policy while decreasing it for those that are more likely. As a result, it downweights the impact of state-action pairs that are overrepresented in training batches while amplifying the impact of those that are underrepresented. This helps break spurious correlations introduced by the policy, allowing agents to focus on the true causal factors. Section \ref{sec:advatange_example} provides a detailed numerical example illustrating this effect.

Our experiments on the Key2Door, Frozen T-Maze, and Diversion environments confirmed that training on advantages leads to more robust agents that generalize better out-of-trajectory. Furthermore, as explained in Section \ref{sec:practical_considerations}, our empirical results reveal how implementation choices, such as batch size and advantage normalization, affect the learned representations. Finally, the KL-divergence analysis of the action probabilities further demonstrates that using the advantage function makes agents more reliant on causal factors.





\bibliography{main}
\bibliographystyle{apalike}

\newpage

\newpage
\appendix

\section{Proofs}\label{ap:proofs}
\theoremone*
\begin{proof}

    \begin{equation}
\begin{aligned}
    A^\pi_\Phi(s_t, a_t) &= Q^\pi(s_t, a_t) - V^\pi(\Phi(s_t)) \\
    &= Q^\pi(s_t, a_t) - \sum_{s'_t \in \{s_t\}^\Phi} P^\pi(s'_t \mid \Phi(s_t)) V^\pi(s'_t)\\
    &= Q^\pi(s_t, a_t) - \sum_{s'_t \in \{s_t\}^\Phi} P^\pi(s'_t \mid \Phi(s_t))\sum_{a' \in \mathcal{A}} \pi(a' \mid s'_t) Q^\pi(s'_t, a') \\
    & = Q^\pi(s_t, a_t) - \sum_{\substack{s'_t \in \{s_t\}^\Phi, a'_t \in \mathcal{A}}} P^\pi(s'_t, a'_t \mid \Phi(s_t)) Q^\pi(s'_t, a'_t) \\
    & = (1 - P^\pi(s_t, a_t \mid \Phi(s_t))) Q^\pi(s_t, a_t) - \sum_{ s'_t, a'_t  \neq  s_t, a_t } P^\pi(s'_t, a'_t \mid \Phi(s_t)) Q^\pi(s'_t, a'_t) \\
    &= (1 - P^\pi(s_t, a_t \mid \Phi(s_t))) Q^\pi(s_t, a_t) - (1 - P^\pi(s_t, a_t \mid \Phi(s_t))) \tilde{Q}^\pi(\neg \langle s_t, a_t \rangle)) \\
    &= (1 - P^\pi(s_t, a_t \mid \Phi(s_t)))( Q^\pi(s_t, a_t) - \tilde{Q}^\pi(\neg \langle s_t, a_t \rangle)).
\end{aligned}
\end{equation}
\normalsize
since

\begin{equation}
\begin{aligned}
    \tilde{Q}^\pi(\neg \langle s_t, a_t \rangle)) &= \frac{\sum_{ s'_t, a'_t  \neq  s_t, a_t } P^\pi(s'_t, a'_t \mid \Phi(s_t)) Q^\pi(s'_t, a'_t)}{\sum_{ s'_t, a'_t  \neq  s_t, a_t } P^\pi(s'_t, a'_t \mid \Phi(s_t))} \\ 
    &= \frac{\sum_{ s'_t, a'_t \neq s_t, a_t } P^\pi(s'_t, a'_t \mid \Phi(s_t))Q^\pi(s'_t, a'_t)}{1 - P^\pi(s_t, a_t \mid \Phi(s_t))}.
\end{aligned}
\end{equation}
\end{proof}

\section{Experimental results}\label{ap:experimental_results}
\subsection{Results using REINFORCE}\label{ap:reinforce}
Figure \ref{fig:qvsa_reinforce} shows the performance using REINFORCE in the training and evaluation variants of the three environments. 
\begin{figure}[H]
    \centering
    \begin{subfigure}[t]{0.32\textwidth}
        \centering
        \includegraphics[width=\textwidth]{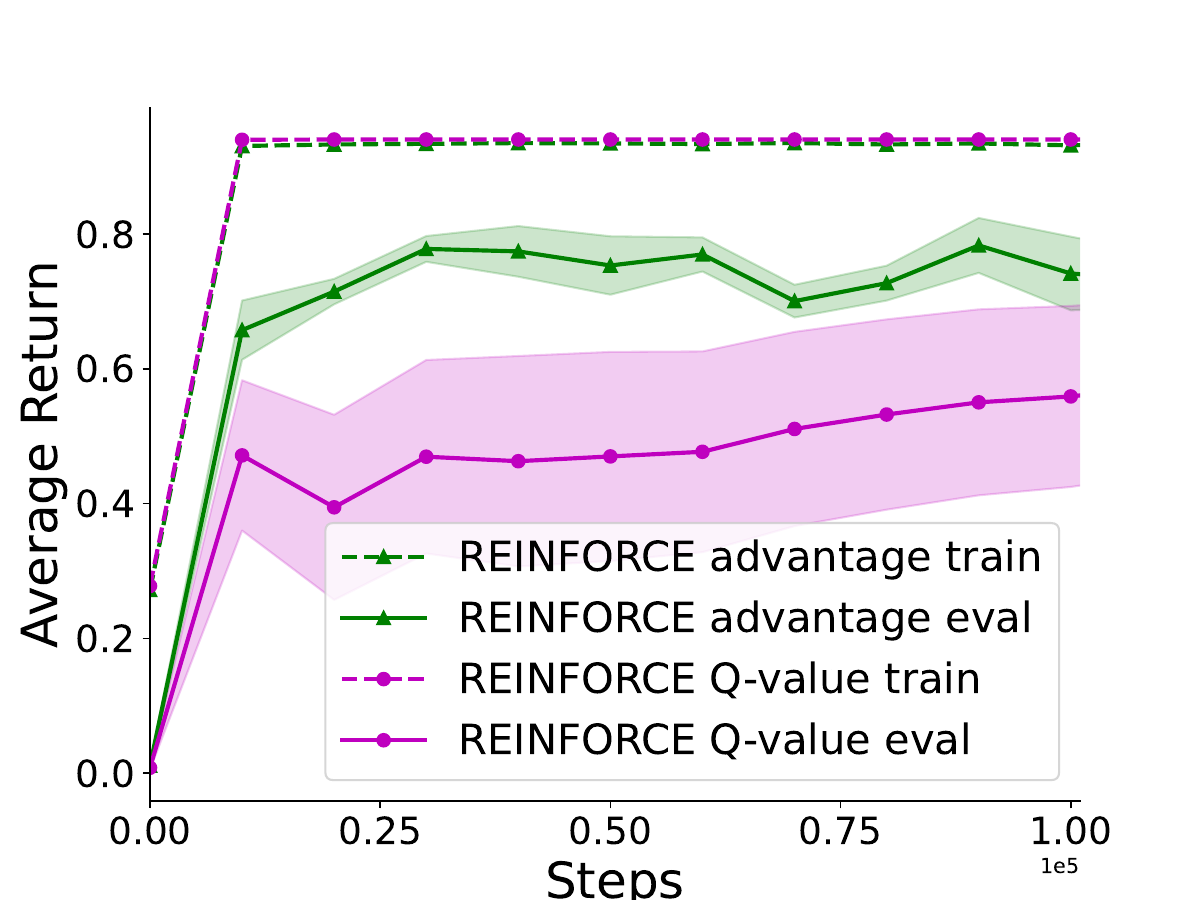}
    \end{subfigure}
    \hfill
    \begin{subfigure}[t]{0.32\textwidth}
        \centering
        \includegraphics[width=\textwidth]{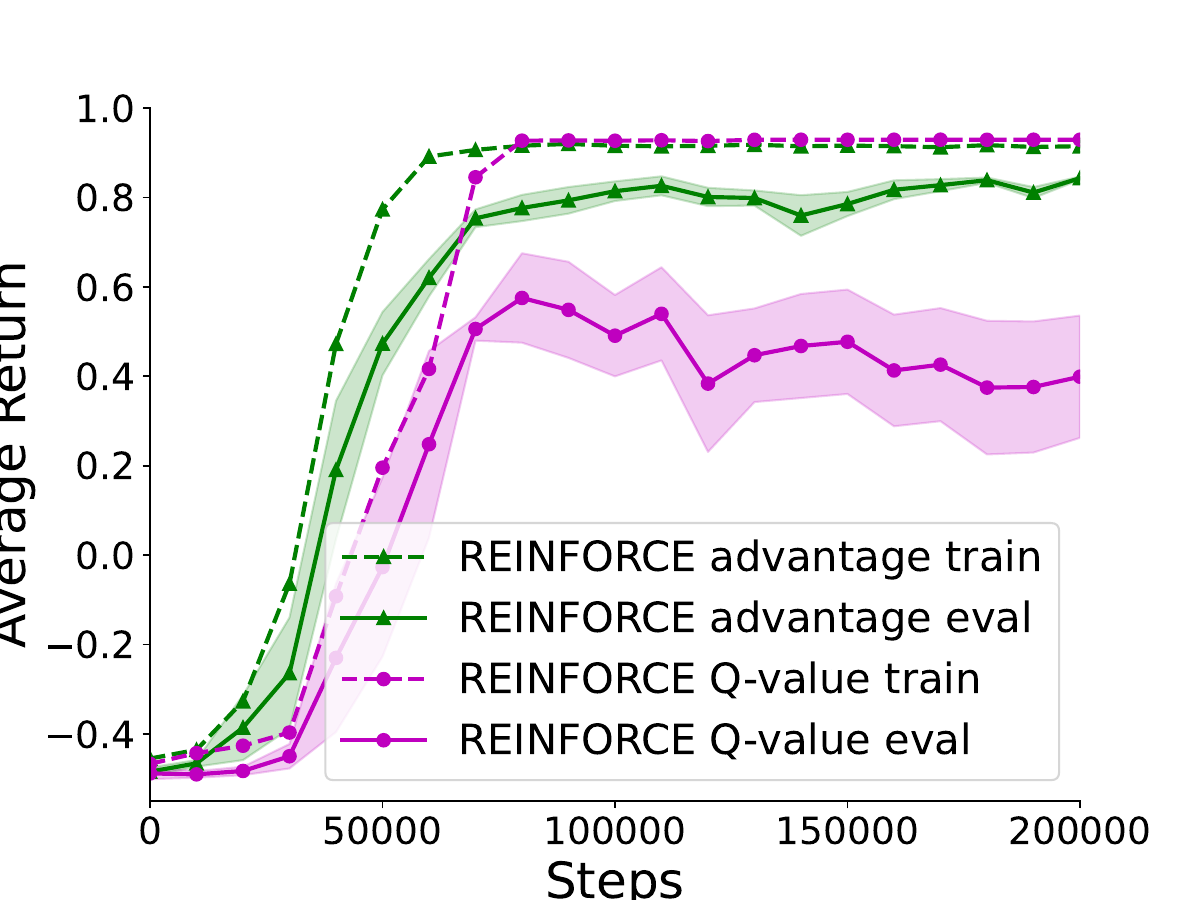}
    \end{subfigure}
    \hfill
    \begin{subfigure}[t]{0.32\textwidth}
        \centering
        \includegraphics[width=\textwidth]{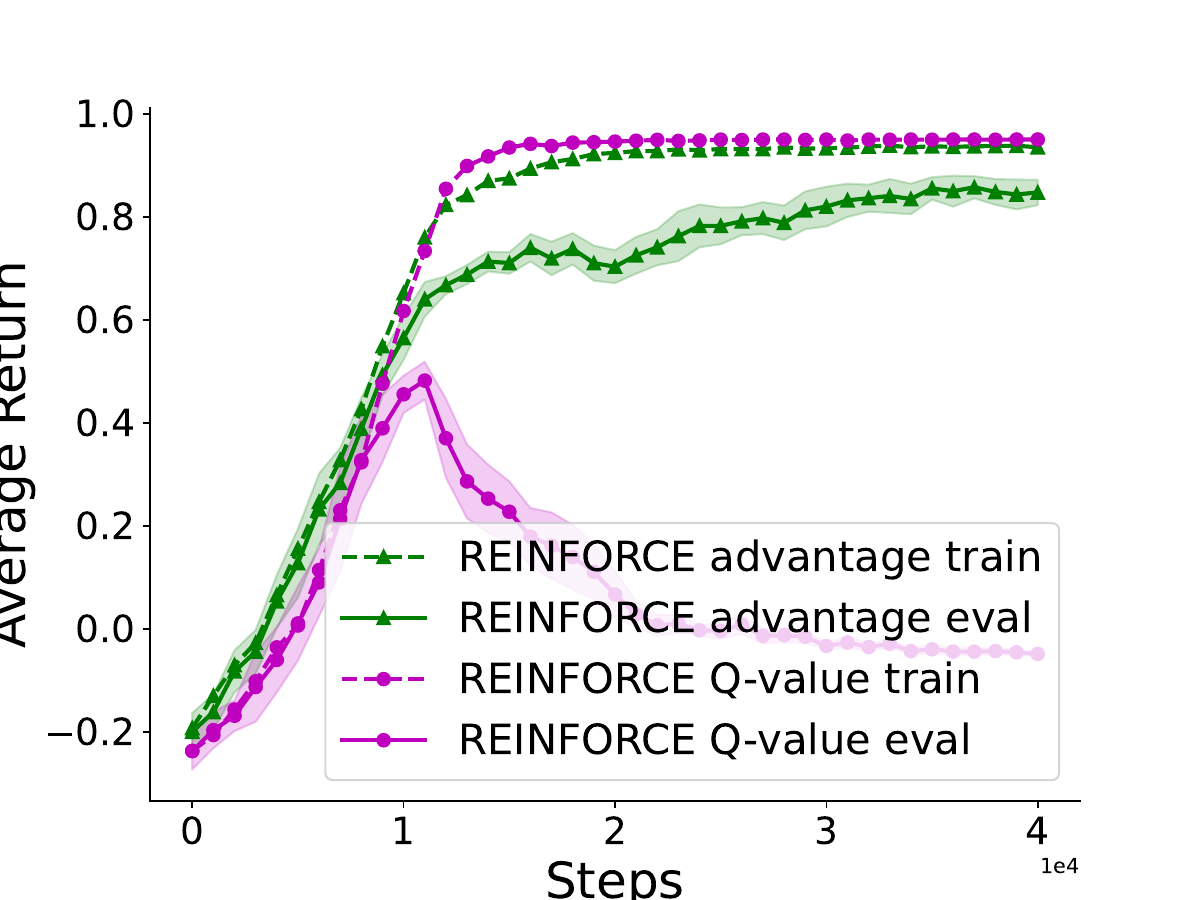}
    \end{subfigure}
    \hfill
    \caption{Performance of REINFORCE using Q-values and Advantages in both the training and evaluation variants of the Key2Door (first plot), Frozen T-Maze (second plot), and Diversion (third plot) environments.
}
\vspace{-10pt}
    \label{fig:qvsa_reinforce}
\end{figure}
\subsection{Advantage normalization}\label{ap:advantage_normalization}
Figure \ref{fig:normalize_all} compares the performance of PPO with and without advantage normalization in the evaluation variants of the three environments.

\begin{figure}[H]
    \centering
    \begin{subfigure}[t]{0.32\textwidth}
        \centering
        \includegraphics[width=\textwidth]{figures/keydoor_normalization.pdf}
    \end{subfigure}
    \begin{subfigure}[t]{0.32\textwidth}
        \centering
        \includegraphics[width=\textwidth]{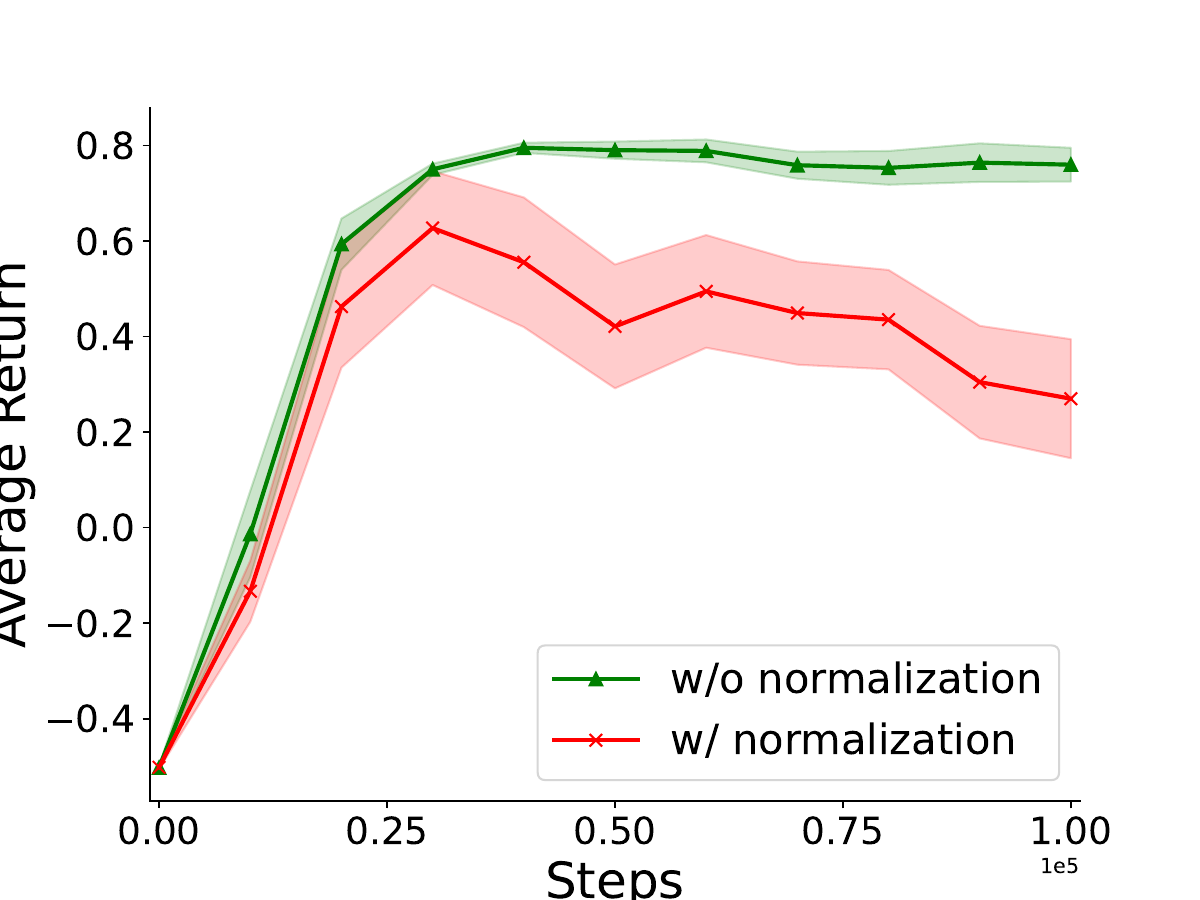}
    \end{subfigure}
    \begin{subfigure}[t]{0.32\textwidth}
        \centering
        \includegraphics[width=\textwidth]{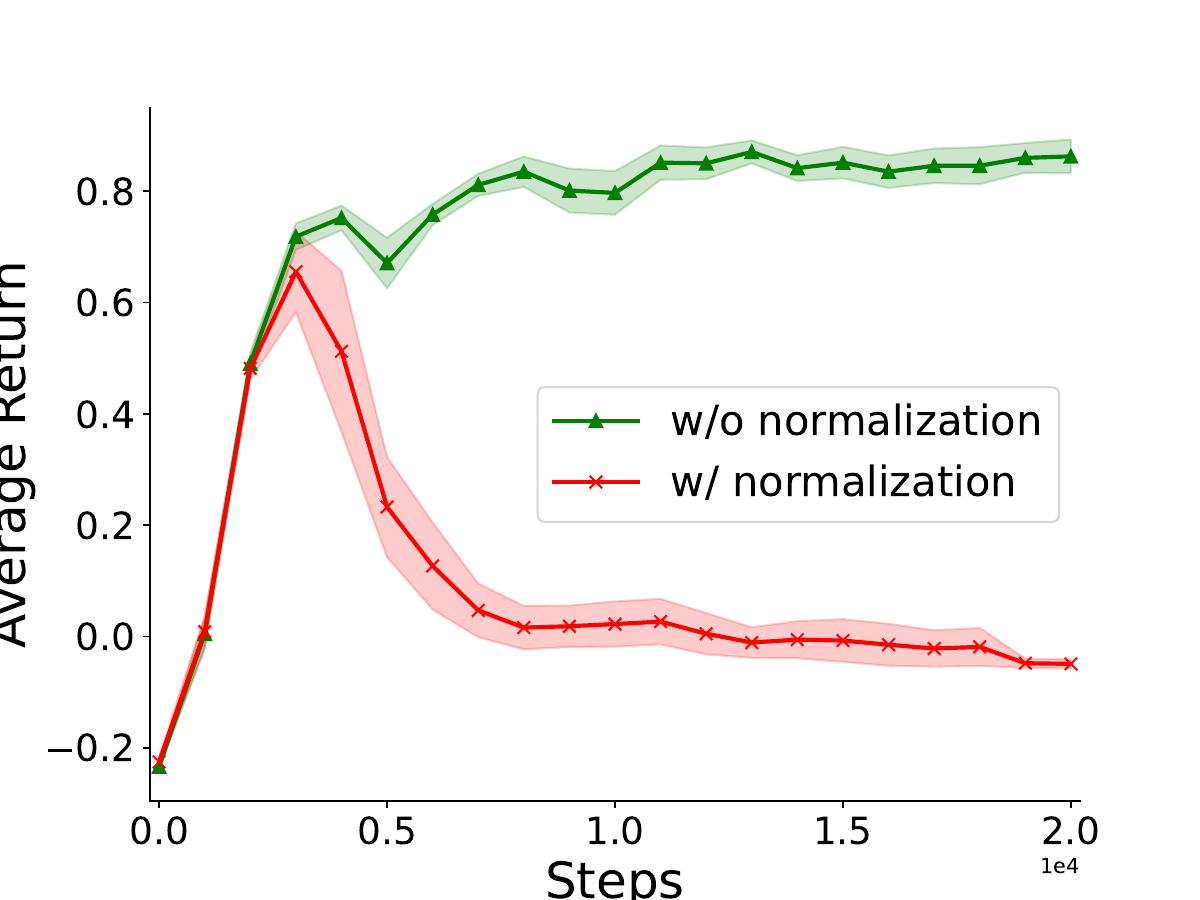}
    \end{subfigure}
    \hfill
    \caption{Performance of PPO with and without advantage normalization in the evaluation variants of the Key2Door (first plot), Frozen T-Maze (second plot), and Diversion (third plot) environments.}
\vspace{-10pt}
    \label{fig:normalize_all}
\end{figure}
\subsection{Batch size}\label{ap:batch_size}
Figure \ref{fig:batch_size_all} compares the performance of PPO with different batch sizes in the evaluation variants of the three environments. While increasing the batch size seems to help in the Key2Door and Diversion environments, it has little effect in the Frozen T-Maze environment.
\begin{figure}[ht]
    \centering
    \begin{subfigure}[t]{0.32\textwidth}
        \centering
        \includegraphics[width=\textwidth]{figures/keydoor_batchsize.pdf}
    \end{subfigure}
    \hfill
    \begin{subfigure}[t]{0.32\textwidth}
        \centering
        \includegraphics[width=\textwidth]{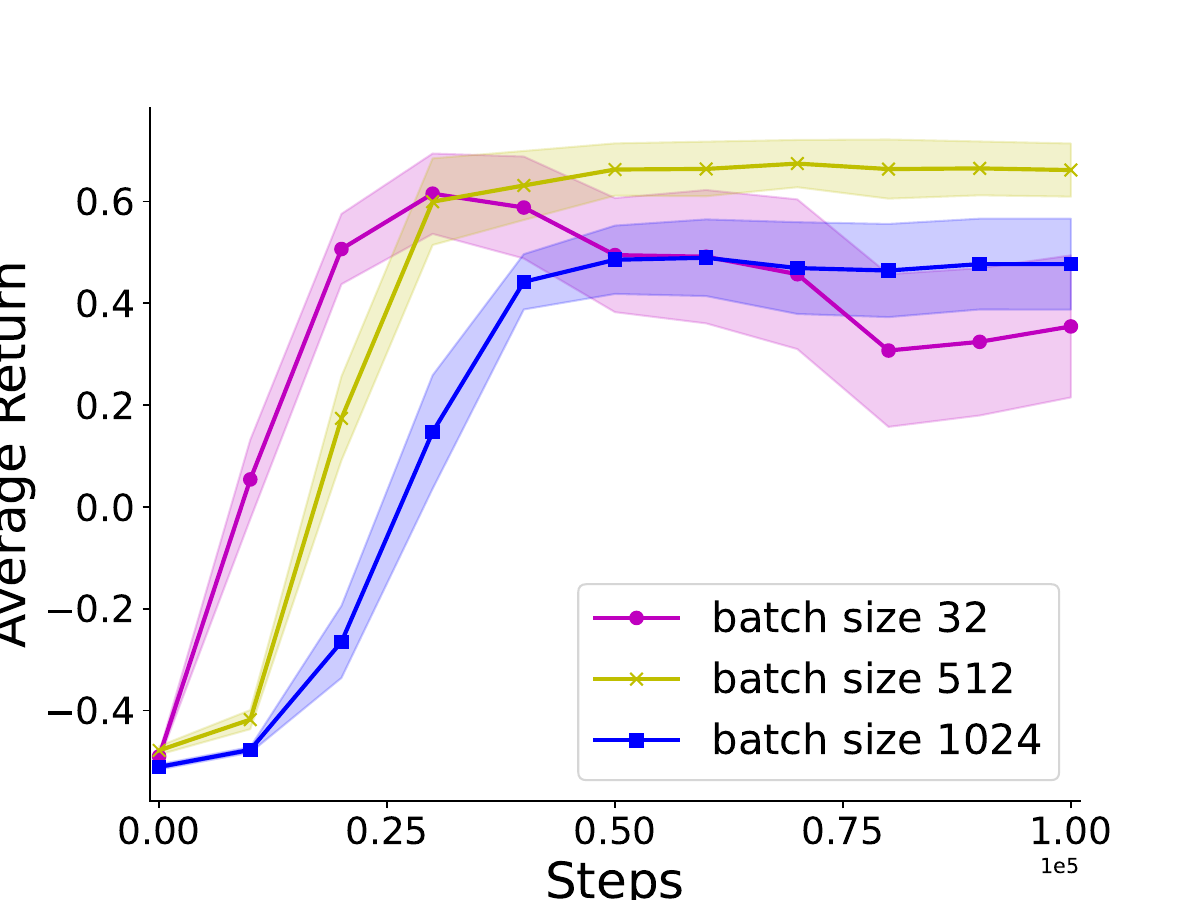}
    \end{subfigure}
    \hfill
    \begin{subfigure}[t]{0.32\textwidth}
        \centering
        \includegraphics[width=\textwidth]{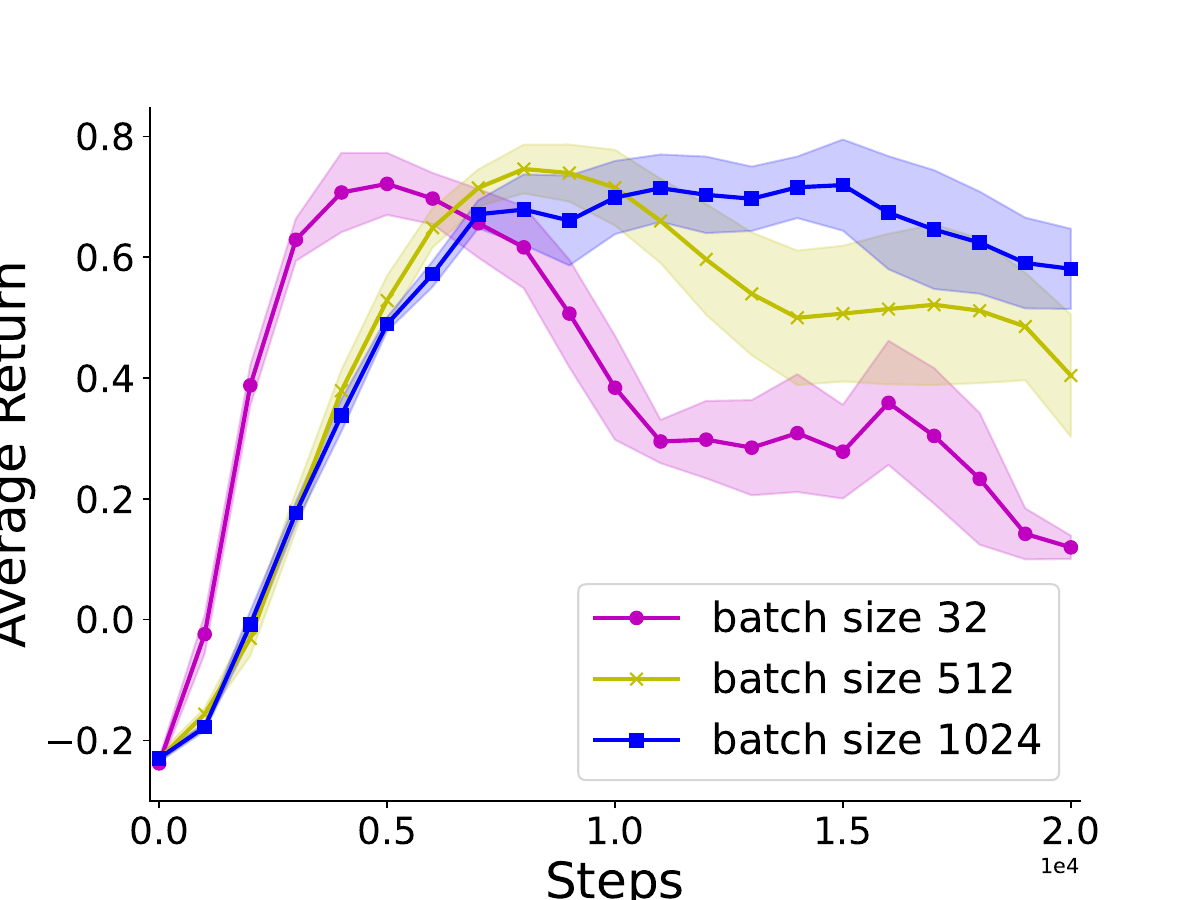}
        \label{fig:plot2}
    \end{subfigure}
    \caption{Performance of PPO with different batch sizes in the evaluation variants of the Key2Door (first plot), Frozen T-Maze (second plot), and Diversion (third plot) environments.}
\vspace{-10pt}
    \label{fig:batch_size_all}
\end{figure}
\subsection{Policy KL divergence}\label{ap:heatmaps}

\paragraph{Frozen T-Maze}
Figure \ref{fig:heatmaps_tmaze} shows the KL divergence of action probabilities in the Frozen T-Maze environment, produced by the policy when the signal is either purple or green, measured at different training steps (top: 10k steps, middle: 50k steps, bottom: 100k steps). To compute these divergences, we take the observation stack received by the agent and query the policy network twice: once with the original stack, and once with the signal bit flipped to the opposite value. The KL divergence at each cell is the average over 100 evaluation episodes.

The left heatmaps reveal that after 100k training steps, the agent trained on the $Q$-value largely ignores the signal, except at the starting location. In contrast, the agent trained on the advantage function conditions its action choices on the signal value throughout the maze.

Interestingly, the left heatmaps also reveal that as training progresses, the trajectories followed by the $Q$-value-trained agent become increasingly deterministic. By the end of training, the agent consistently chooses the top path when the signal is green and the bottom path when the signal is purple. In contrast, the agent trained on the advantage function continues to follow a diverse set of trajectories and does not exhibit a strong preference for any particular path. Note that multiple optimal paths exist for each signal value.
\begin{figure}[H]
    \centering
    \begin{subfigure}[t]{0.45\textwidth}
        \centering
        \includegraphics[width=\textwidth]{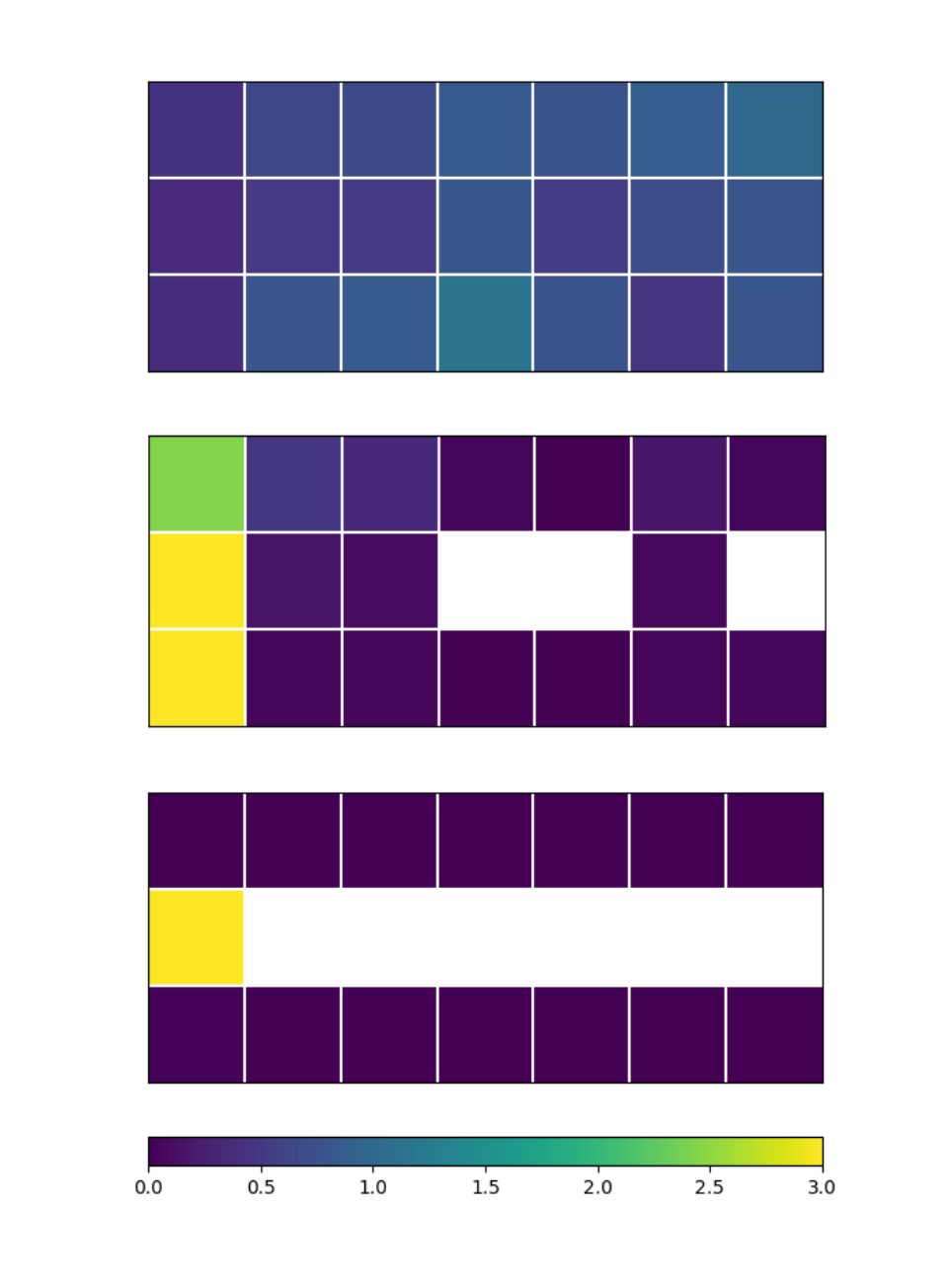}
    \end{subfigure}
    \begin{subfigure}[t]{0.45\textwidth}
        \centering
        \includegraphics[width=\textwidth]{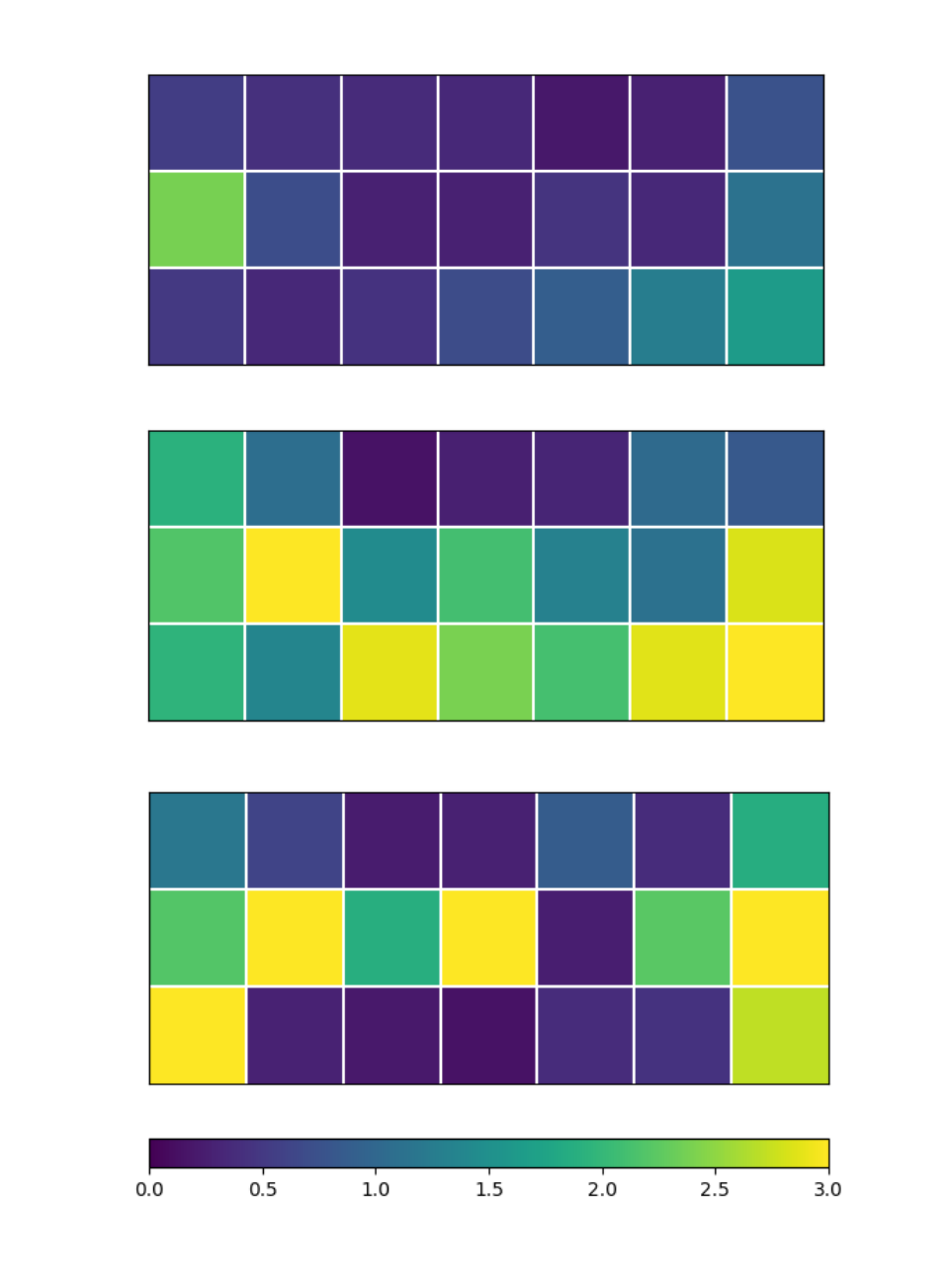}
    \end{subfigure}
    \vspace{-10pt}
    \caption{KL divergence of action probabilities in the Frozen T-Maze environment at different training steps (top 10k steps, middle 50k steps, and bottom 100k steps) for agents trained on the $Q$-value (left) and the Advantage function (right).}
    \label{fig:heatmaps_tmaze}
\end{figure}

\paragraph{Diversion}

Figure \ref{fig:KL_diversion} shows the policy KL divergence across each column in the Diversion environment, measured at different training steps (top: 3k steps, middle: 10k steps, bottom: 20k steps). For each column, the divergence is computed by comparing the action probabilities output by the agent’s policy when positioned in the top versus bottom row.

The left heatmaps indicate that the agent trained on the $Q$-value learns to ignore the bit indicating the row after 10k training steps. In contrast, the agent trained on the advantage function continues to use the row information when deciding which action to take.

\begin{figure}[H]
    \centering
    \begin{subfigure}[t]{0.45\textwidth}
        \centering
        \includegraphics[width=\textwidth]{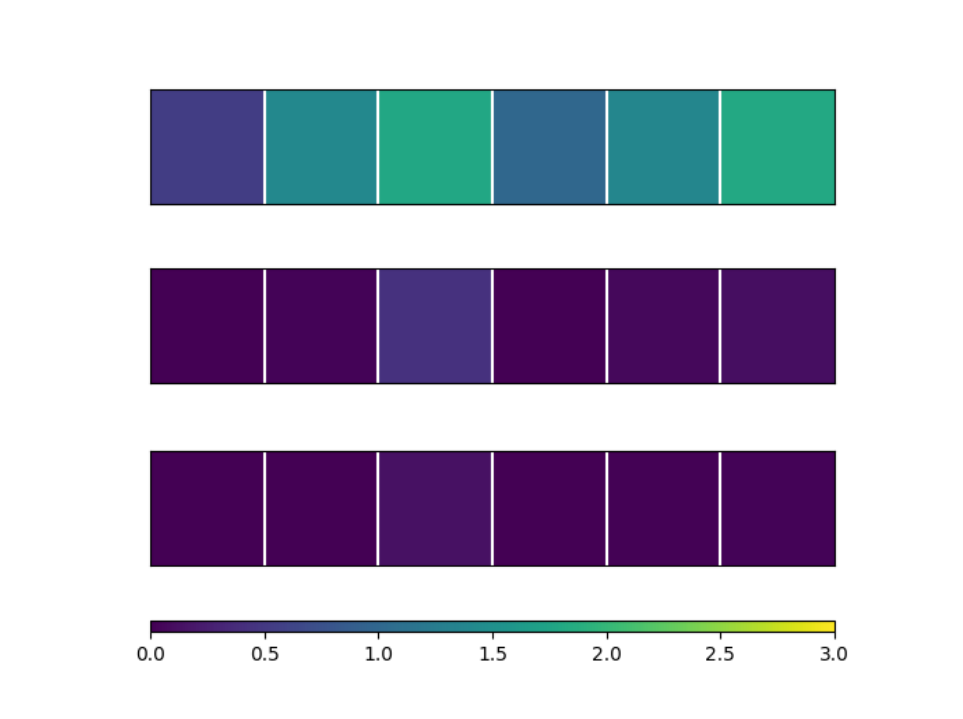}
    \end{subfigure}
    \begin{subfigure}[t]{0.45\textwidth}
        \centering
        \includegraphics[width=\textwidth]{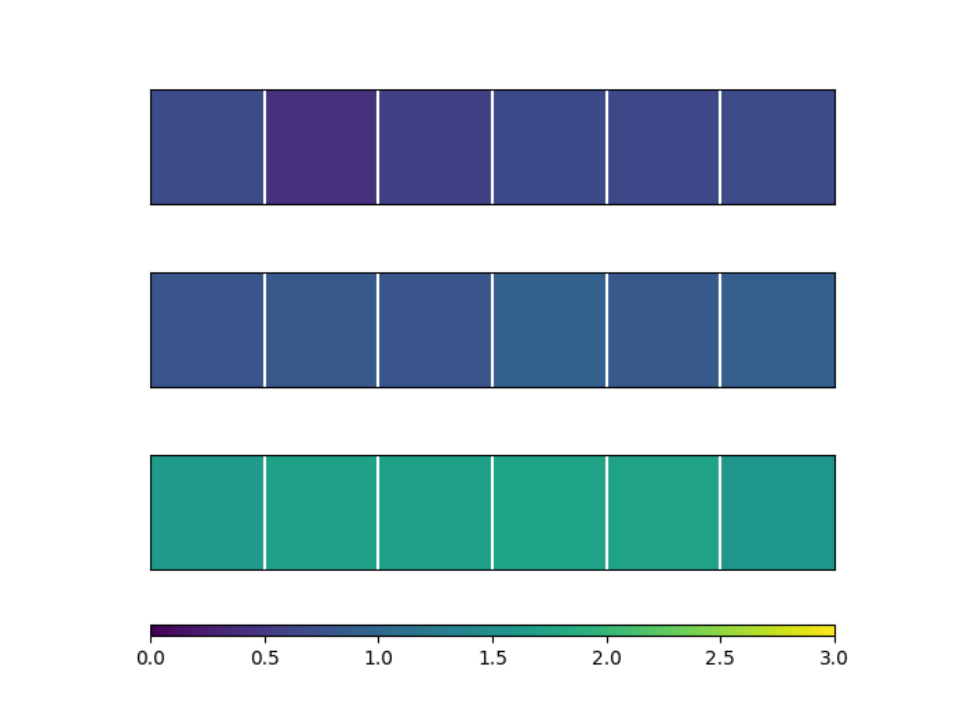}
    \end{subfigure}
    \caption{KL divergence of action probabilities in the Diversion environment measured at different training steps (top 3k steps, middle 10k steps, and bottom 20k steps) for agents trained on the $Q$-value (left) and the Advantage function (right).}
    \label{fig:KL_diversion}
\end{figure}

\section{Environments}\label{ap:environments}
\begin{figure}[H]
     \centering
     \begin{subfigure}[b]{0.33\textwidth}
         \centering    \includegraphics[width=\textwidth]{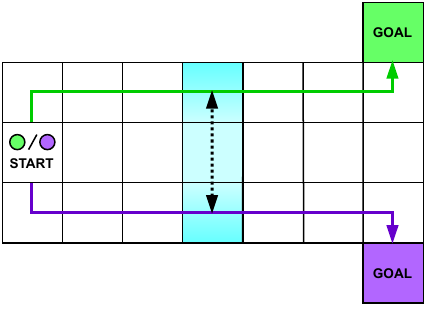}
     \end{subfigure}
     \hspace{40pt}
     \begin{subfigure}[b]{0.45\textwidth}
         \centering
\includegraphics[width=0.95\textwidth]{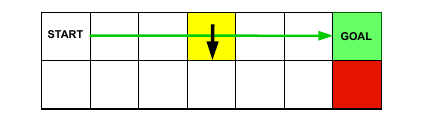} 
    \vspace{22pt}
     \end{subfigure}
      \caption{Illustrations of the Frozen T-Maze (left) and Diversion (right) environments.}
      \label{fig:frozen_offtrack}
      
\end{figure}
\paragraph{Frozen T-Maze}
This environment is a variant of the popular T-Maze setting \citep{bakker2001reinforcement}. At the starting location, the agent receives a binary signal: either green or purple. Its task is to navigate to the right and reach the correct goal at the end of the maze. The agent receives a reward of $+1$ for reaching the green (purple) goal when the green (purple) signal was observed, and a penalty of $-1$ otherwise. Additionally, a time penalty of $-0.01$ is applied at each timestep to encourage the agent to take the most direct path to the goal.

At every timestep, the agent observes its current location within the maze, represented by a one-hot-encoded vector. However, the initial signal, represented by a binary variable, is only provided to the agent at the starting location. Crucially, the agent is capable of remembering past observations. When moving randomly, it must retain the initial signal throughout its trajectory. However, once it learns the shortest path to each goal (illustrated by the green and purple arrows), the agent can safely disregard the initial signal. This is because the agent can infer the signal based on its location: if it is on the green (purple) path, it must have received the green (purple) signal. Note that the two paths highlighted in Figure~\ref{fig:frozen_offtrack} are not the only optimal ones. However, for the agent to disregard the initial signal, the paths must not overlap. 

We train agents in the original environment and evaluate them in a modified version where an icy surface (shown in blue) is introduced in the middle of the maze. This ice causes the agent to slip between the upper and lower cells.

\paragraph{Diversion}
In this environment, the agent must move from the start state to the goal state shown in Figure~\ref{fig:frozen_offtrack} (right). A reward of $+1$ is given for reaching the goal, and a penalty of $-1$ is incurred if the agent reaches the red cell instead. As in the other environments, there is a per-timestep penalty of $-0.01$. Observations are 8-dimensional binary vectors: the first 7 elements indicate the column where the agent is located, and the last element indicates the row.

After the agent learns the optimal policy (shown by the green arrow), it can ignore the last element of the observation vector. This is because the optimal policy never visits the bottom row. We train the agent in the original environment and evaluate it in a modified version containing a yellow diversion sign in the middle of the maze, which forces the agent to move to the bottom row.

\section{Experimental setup}\label{ap:experimental_setup}

The experiments were run on a laptop equipped with an Apple M2 Pro processor (12 cores) and 16 GB of RAM. Each run took less than 5 minutes and used at most 2\% of the total RAM.

Agents were trained using Stable Baselines3 \citep{stable-baselines3}. The hyperparameters are listed in Tables~\ref{tab:ppo} (PPO) and~\ref{tab:reinforce} (REINFORCE). For PPO, we adopted the hyperparameters used by \citet{schulman2017proximal} in their Atari experiments, except for the learning rate, which we increased to 1.0e-3 to accelerate convergence. We implemented a minimal version of the REINFORCE algorithm with only three hyperparameters (learning rate, discount factor, and entropy coefficient), for which we used the same values as in PPO. For the Frozen T-Maze, we used a stack of the past 30 observations as input, since solving the task requires memory.
\begin{center}
\vspace{-10pt}
  \begin{table}[h]
  \centering
  \caption{PPO hyperparameters.}
  \vspace{5pt}
\resizebox{0.7\textwidth}{!}{
  \begin{tabular}{ p{5cm}|p{4cm}}
 Rollout steps & 128 \\
 Batch size & 32 \\
 Learning rate & 1.0e-3 \\
 Number epoch & 3 \\
 Discount $\gamma$ & 0.99 \\
 GAE $\lambda$ & 0.95 \\
 Entropy coefficient & 1.0e-2 \\
 Clip  range & 0.1 \\
 Value coefficient & 1 \\
 Number Neurons 1st layer & 128 \\
 Number Neurons 2nd layer & 128 \\
\end{tabular}
}
\vspace{-20pt}
\label{tab:ppo}
\end{table}
\end{center}

\begin{center}
\vspace{-10pt}
  \begin{table}[ht]
  \centering
  \caption{REINFORCE hyperparameters.}
  \vspace{5pt}
\resizebox{0.7\textwidth}{!}{
  \begin{tabular}{ p{5cm}|p{4cm}}
 Learning rate & 1.0e-3 \\
 Discount $\gamma$ & 0.99 \\
 Entropy coefficient & 1.0e-2 \\
 Number Neurons 1st layer & 128 \\
 Number Neurons 2nd layer & 128 \\
\end{tabular}
}
\vspace{-20pt}
\label{tab:reinforce}
\end{table}
\end{center}

\end{document}